\documentclass{article}

\usepackage{microtype}
\usepackage{graphicx}
\usepackage{subfigure}
\usepackage{booktabs} 

\usepackage{hyperref}

\usepackage[accepted]{icml2023}

\usepackage{amsmath}
\usepackage{amssymb}
\usepackage{mathtools}
\usepackage{amsthm}

\usepackage{bbm}
\usepackage{verbatim}

\usepackage[capitalize,noabbrev]{cleveref}

\theoremstyle{plain}
\newtheorem{theorem}{Theorem}[section]
\newtheorem{proposition}[theorem]{Proposition}
\newtheorem{lemma}[theorem]{Lemma}

\theoremstyle{definition}

\theoremstyle{remark}

\newcommand{\pss}{\mathcal{Z}_\mathrm{+}}
\newcommand{\nss}{\mathcal{Z}_\mathrm{-}}
\newcommand{\sz}{\mathcal{Z}}
\newcommand{\ps}{z_\mathrm{+}}
\newcommand{\ns}{z_\mathrm{-}}
\newcommand{\params}{\theta}
\newcommand{\prs}{\Theta}
\newcommand{\grad}{g}
\newcommand{\gest}{\hat{g}}
\newcommand{\gp}{g_\mathrm{+}}
\newcommand{\gn}{g_\mathrm{-}}
\newcommand{\prv}{B}
\newcommand{\prvp}{b}
\newcommand{\x}{v}
\newcommand{\data}{x}
\newcommand{\h}{h}
\newcommand{\lr}{\eta}
\newcommand{\pl}{T}
\newcommand{\trv}{\omega}
\newcommand{\trvp}{\mu}
\newcommand{\tp}{g_\mathrm{+}(\params, \ps)}
\newcommand{\tn}{g_\mathrm{-}(\params, \ns)}
\newcommand{\as}{\mathcal{Z}}
\newcommand{\ac}{z}

\newcommand{\p}{p}
\newcommand{\gs}{G}
\newcommand{\rn}{\tau}
\newcommand{\extra}{\Gamma}
\newcommand{\transient}{s}
\newcommand{\fp}{f_+(\params, \ps)}
\newcommand{\fn}{f_-(\params, \ns)}
\newcommand{\sigp}{\Sigma_+}
\newcommand{\sign}{\Sigma_-}
\newcommand{\sigpn}{\Sigma_{+-}}
\newcommand{\mup}{\mu_+}
\newcommand{\mun}{\mu_-}
\newcommand{\vi}{j}
\newcommand{\tac}{\tilde{\ac}}
\newcommand{\pls}{T_s}

\newcommand{\CO}{\mathcal{O}}
\newcommand{\CC}{\mathcal{C}}
\newcommand{\dee}{\mathrm{d}}
\newcommand{\E}{\mathbb{E}}

\usepackage[textsize=tiny]{todonotes}

\icmltitlerunning{Flexible Phase Dynamics for Bio-Plausible Constrastive Learning}

\begin{document}

\twocolumn[
\icmltitle{Flexible Phase Dynamics for Bio-Plausible Contrastive Learning}




\begin{icmlauthorlist}
\icmlauthor{Ezekiel Williams}{aff1,aff2}
\icmlauthor{Colin Bredenberg}{aff1,aff2}
\icmlauthor{Guillaume Lajoie}{aff1,aff2}
\end{icmlauthorlist}

\icmlaffiliation{aff1}{Department of Mathematics and Statistics, Universite de Montreal, Quebec, Canada}
\icmlaffiliation{aff2}{Mila, Quebec AI Institute, Quebec, Canada}

\icmlcorrespondingauthor{Ezekiel Williams}{ezekiel.williams@mila.quebec}
\icmlcorrespondingauthor{Guillaume Lajoie}{g.lajoie@umontreal.ca}

\icmlkeywords{Machine Learning, Neuroscience, Computational Neuroscience, Probabilistic Models, Energy Based Models, Contrastive Learning, Biologically Plausible Learning, Neuromorphic Computing}

\vskip 0.3in
]



\printAffiliationsAndNotice{}  

\begin{abstract}
Many learning algorithms used as normative models in neuroscience or as candidate approaches for learning on neuromorphic chips learn by contrasting one set of network states with another. These Contrastive Learning (CL) algorithms are traditionally implemented with rigid, \emph{temporally non-local}, and periodic learning dynamics that could limit the range of physical systems capable of harnessing CL. In this study, we build on recent work exploring how CL might be implemented by biological or neurmorphic systems and show that this form of learning can be made temporally local, and can still function even if many of the dynamical requirements of standard training procedures are relaxed. Thanks to a set of general theorems corroborated by numerical experiments across several CL models, our results provide theoretical foundations for the study and development of CL methods for biological and neuromorphic neural networks.
\end{abstract}

\section{Introduction}
\emph{Contrastive learning} (CL) represents a powerful family of algorithms for learning representations, spanning a broad range of machine learning methodologies including generative and discriminative modelling, as well as supervised and unsupervised learning. Broadly, CL can be divided into \emph{equilibrium-based} methods, that use a (stochastic) dynamical process to calculate internal network states, and non-equilibrium methods, that calculate internal states directly in a single pass through the network. CL has been proposed as a normative model for biological learning \cite{baldi1991contrastive,illing2021local,hinton1984boltzmann,scellier2017equilibrium,cao2020characterizing}, but requires highly structured training dynamics that restrict the kinds of biological systems that could implement this form of learning. 
Here, we investigate alternative training dynamics for CL, to better understand the breadth of systems capable of implementing it. 

CL learns representations by leveraging the statistical differences between \emph{positive samples} (samples from the training set) and \emph{negative samples} (artificially or internally generated data). We restrict our study to the subclass of CL algorithms with a two-term gradient, where one term depends on positive but not negative samples, and the other term on negative but not positive. This \emph{two-term} CL class comprises a wide range of algorithms, including most algorithms of interest to the neuroscience community. For example, the Boltzmann machine \cite{ackley1985learning}, equilibrium propagation \cite{scellier2017equilibrium}, and energy based models optimizing the negative log-likelihood \cite{lecun2006tutorial} all fall within this subclass. Conversely, SimCLR \cite{chen2020simple} is an example of a non-two-term CL method. 

Two-term CL algorithms have a periodic, biphasic time course of learning. First, during the \emph{positive phase}, the model calculates internal states given a positive sample. This is followed by a \emph{negative phase} where neural responses to a negative sample are computed. For example, in a Boltzmann machine \cite{ackley1985learning} the positive phase calculates network states conditioned on data while the negative phase calculates internally generated states free from data conditioning; in noise contrastive estimation, the positive phase calculates network states in response to real data while the negative phase calculates responses to noise \cite{gutmann2010noise}. Finally, network responses from each phase are used to calculate the gradient and a single step of gradient descent is taken before repeating the whole process for the next learning iteration \cite{ackley1985learning, scellier2017equilibrium, hinton2022forward} (Fig.\ref{fig1}.A). As such, CL suffers from constraints on training dynamics that restrict the set of physical neural networks capable of running the algorithm. In this work, we address four of these constraints and show they can be considerably relaxed:
\begin{enumerate}
    \item Temporal non-locality: the use of information from past network states that is too old to be available to the network without auxiliary memory mechanisms 
    \item Strict phasic periodicity: oscillation between two distinct phases
    \item Rapidly tuneable learning rates: precise modulation of network plasticity during training
    \item Deterministic phase length: keeping the same phase length for all phases during training
\end{enumerate}

\paragraph{Contributions:}
In this paper we address temporal non-locality and periodicity for all two-term CL methods, and investigate how variability of phase length and limited learning rate modulation affect equilibrium-based CL. We propose an importance sampling-inspired approach to estimating the gradient that makes two-term CL methods temporally local and show it is unbiased (see Appendix \S \ref{sec:bias-variance}). It does so by stochastically selecting either a single positive or single negative sample to learn from at each gradient step, thus also eliminating the strict periodicity seen in past algorithms. The proposed approach provides a principled way to tune the amount of time spent taking gradient steps with respect to positive versus negative samples. Interestingly, we find that it is not always optimal to spend an equal amount of training time learning from positive compared to negative samples. Next, we demonstrate quite generally that equilibrium-based methods, and thus equilibrium-based CL, can still learn effectively--with asymptotically unbiased gradient estimates--even with no learning rate modulation (i.e. avoiding the wait until the end of phases to update parameters) and with a significant amount of noise in the length of each phase. Taken together, our results formally relax the previous requirements on the learning dynamics of many CL algorithms, thereby demonstrating that a broader set of systems might be capable of implementing this form of learning.

\section{Background}

\paragraph{Bio-Plausible Learning:}
How the brain learns is a fundamental scientific question driving research not only in neuroscience but also in AI. In particular, AI algorithms are being used more and more as \emph{normative models} of learning in biological neural networks, a line of study that has already yielded exciting results \cite{richards2019deep, sorscher2022unified, hennequin2018dynamical, payeur2021burst}. However, many AI algorithms exhibit properties that go beyond gross abstractions of neural processing to be computationally incompatible with neural hardware. This means that an important component of AI-inspired modelling in neuroscience is teasing out which aspects of a candidate learning model cannot be implemented by the brain and whether the model might be modified so as to make it more plausibly implementable by biological circuits (\emph{bio-plausible}). For example, backpropagation computes using biologically implausible mechanisms such as \emph{spatial non-locality} \cite{diederich1987learning} and \emph{weight-transport} \cite{lillicrap2016random}.

The study of bio-plausible learning is not solely applicable to neuroscience but also has relevance to AI in itself, as algorithms that function like the brain often have computational benefits \cite{lake2017building}. An especially important example of this synergy between bio-plausible learning and AI development is neuromorphic computing: neuromorphic hardware represents a potential solution to AI energy efficiency problems \cite{strubell2020energy}, but often puts similar constraints on algorithms that biological systems do, for example utilizing local storage of memory \cite{schuman2022opportunities, frenkel2021bottom}.

\paragraph{Contrastive Learning and Bio-Plausibility:}
Often the positive and negative phases of CL are thought of, from a psychological modelling perspective, as being implemented during wakeful processing of stimuli and periods of sleep, respectively \cite{hinton2022forward, crick1983function}, just as in the Wake-Sleep algorithm (another classic learning model) \cite{dayan1995helmholtz}. Other studies have related these phases to oscillations in cortical circuits \cite{baldi1991contrastive} or have suggested that the difference between phases could be delineated by saccades in the visual system \cite{illing2021local}. Given a lack of hard evidence in favour of these hypotheses, we will refer to these phases solely as positive and negative in what follows.

A subset of CL algorithms do not require backpropagation for neural network-based learning \cite{ackley1985learning, scellier2017equilibrium, hinton2022forward}; it is this feature in particular that has brought attention to CL methods as bio-plausible learning models. However, the learning dynamics for CL methods exhibit properties that have not been observed in the brain.

To begin with, the majority of CL methods are \emph{temporally non-local}. This refers to when an algorithm uses information that is not available within some short distance of the present timestep to update the parameters of the network at the current timestep \cite{illing2021local}. For CL methods, temporal non-locality appears because network responses to positive \emph{and} negative samples must be calculated before performing a parameter update (e.g. Fig.\ref{fig1}.A). Notably, this also introduces a precise positive-negative phase period during the learning time course. The problem with temporal non-locality is that it would require a neural network to save its neuron states in a memory buffer for a non-negligible amount of time before recalling them. Memory storage could be solved in theory by synaptic weight consolidation (introduce a candidate modification in phase 1, and consolidate the change in phase 2), but there's no known relationship between this phenomenon and biphasic network dynamics in the brain.

\paragraph{Equilibrium-Based Contrastive Learning:}
Beyond the problems of temporal non-locality and periodicity, two additional complications exist for \emph{equilibrium-based CL} methods used to train (stochastic) dynamical system neural networks. For each data point observed during training, equilibrium-based CL must run internal dynamics on the space of the given model's hidden unit activations until the dynamics have converged to the vicinity of a fixed point. This means that, for each positive and negative sample pair, learning rates must be set to zero and dynamics run two times before learning rates are turned-on for a single parameter update \cite{ackley1985learning, scellier2017equilibrium, hinton2022forward} (Fig.\ref{fig1}.A). Such a specific pattern of learning rate scheduling and periodic stimulus/sample presentation is as yet unsupported empirically in biological systems: this begs the question whether aspects of these training dynamics can be relaxed. Lastly, because of the convergence requirements on the dynamics of equilibrium-based learning, when training AI methods one usually fixes a particular length of time to run dynamics for each sample. Comparatively, regardless of the proposed implementation of biphasic learning (sleep/wake, neural oscillations, saccades, active/quiescent) animals usually spend a variable amount of time in each state. This motivates the question of how much variability equilibrium-based CL algorithms might tolerate in the length of phases.

\paragraph{Example CL system---The Boltzmann Machine:}

To ground this discussion of biphasic learning, we will provide a brief example of a classic equilibrium-based CL method: the Boltzmann machine \cite{ackley1985learning}. The Boltzmann machine is a generative Recurrent Neural Network (RNN) model that learns a probability distribution over a $d$ dimensional binary support from a dataset $\{\data^{(i)}\}_{i=1}^N$, $\data^{(i)} \in \{0, 1\}^d$. The RNN units $\ac \in \{0,1\}^n$ are partitioned into visible units, $\x_{j}$ $j \in \{0, \dots d\}$, whose learned activity represents samples from the $d$ dimensional support, and hidden units $\h_{i}$, $i \in \{0, \dots, n_h\}$, that encode higher order interactions between the visible units. Then $\ac \in \{0, 1\}^{d + n_h}$ with $\ac = [\x, \h]$. The units are connected via a synaptic weight matrix $\params \in \mathbb{R}^{n\times n}$ that learns to encode the correlation structure between units via the following contrastive Hebbian learning rule:
\begin{align}
    \Delta \params_{kl} = \lr \ac_{k,+}\ac_{l,+} - \lr \ac_{k,-}\ac_{l,-},
\end{align}

where the `$+$' subscript denotes an activation generated from $K \in \mathbb{N}$ steps of dynamics with the visible units clamped to a data point (a positive phase) and the `$-$' subscript denotes an activation generated from $K$ steps of recurrent dynamics with the visible units unclamped (a negative phase). The weight update finally occurs after both phases have been run and requires the temporally non-local saving of activations from the first phase. Learning rates are set to zero while the phase dynamics are run and only turned on for the weight update. These periodic learning dynamics are referred to as $K$-step Contrastive Divergence (CDK) \cite{hinton2002training}, and are visualized in Fig.\ref{fig1}.A. While CDK is not the only algorithm for training a Boltzmann machine, it is typical of the training dynamics used by other methods \cite{fischer2014training}.


\paragraph{Related Work:}

Here, we non-exhaustively review approaches to learning that are equilibrium-based but not contrastive, then contrastive non-equilibrium methods, and finally the intersection of these.

The majority of bio-plausbile, non-contrastive equilibrium-based methods find their roots in Predictive Coding (PC) \cite{srinivasan1982predictive, rao1999predictive}, which was retrospectively shown to be a variational method \cite{bishop2006pattern} for learning a hierarchical Gaussian model \cite{friston2005theory}. Given the prescribed generative model, these works are more narrow in scope than ours, which explores equilibrium learning quite broadly in the context of stochastic gradient descent and fixed point latent dynamics.

There are two non-equilibrium-based CL classes of particular relevance to this study, both inspired by noise contrastive estimation \cite{gutmann2010noise}. The first class originates from Contrastive Predictive Coding (CPC) \cite{oord2018representation}, which contrasts positive samples of the future conditioned on the present with negative samples taken at random from the analyzed time series. A bio-plausible version of CPC \cite{illing2021local} uses the same approach taken in this study to arrive at temporally local and aperiodic learning, but focuses on the specific case of CPC and does not make the link with unbiased gradients and importance sampling. The second class encompasses the Forward-Forward algorithm, and related works, that do binary classification of positive from negative samples at each layer to achieve learning without backpropagation \cite{hinton2022forward, ororbia2023predictive}. To our knowledge, temporally local and aperiodic training of these algorithms has not been studied.

Equilibrium-based CL was first researched in the late 80s and early 90s, with the Boltzmann machine \cite{ackley1985learning} and its deterministic variants \cite{baldi1991contrastive, movellan1991contrastive}. Baldi \& Pineda \cite{baldi1991contrastive} explored temporally local learning in this context and, as in our work, achieve temporal locality by breaking up the two-term gradient into single positive and negative-phase terms. Unlike our work, they consider solely periodic updating---associating the learning process with cortical oscillations---and explain that learning follows gradients only in the small-learning rate limit. Temporally local learning was investigated more recently with two modifications to Equilibrium Propagation (EP) \cite{scellier2017equilibrium, ernoult2020equilibrium, laborieux2022holomorphic}. In the first \cite{ernoult2020equilibrium}, the proposed solution still requires two network passes and, like Baldi \& Pineda, relies on an infinitesimally small learning rate. This is necessary because Ernoult et al.'s solution to temporal non-locality uses continual weight updates during the positive phase. We use a similar argument, but to achieve concurrent learning and latent dynamics, rather than locality. The second study \cite{laborieux2022holomorphic} uses oscillatory clamping for training, similar to the oscillatory teacher signal used historically \cite{baldi1991contrastive}. It does so by leveraging activity in the complex plane, which provides an elegant solution for phase management but requires complex neuron activiations. This complicates the mapping to standard neuroscience and AI models which usually represent neural activity in real spaces. The second study shares another similarity with ours, in that it uses a separation of timescales argument to achieve concurrent learning and latent-space dynamics. However, the separation of timescales is left as an assumption, in contrast to our derivation of error bounds that explicitly account for the two timescales.

\section{Theoretical Results}

\begin{figure}[ht!]
\vskip 0.2in
\begin{center}
\centerline{\includegraphics[width=0.95\columnwidth]{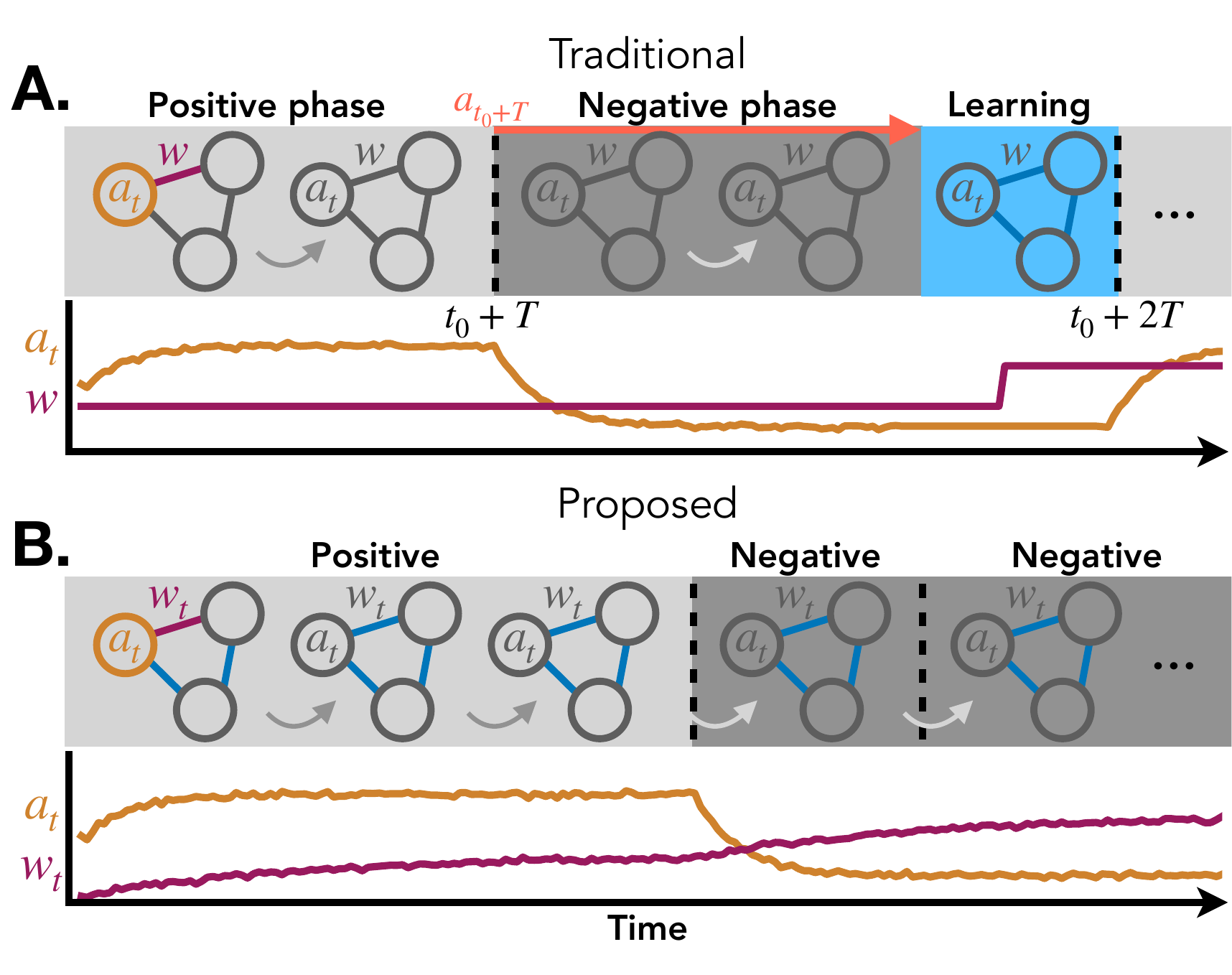}}
\caption{Comparison of learning dynamics for two algorithms. Equilibrium-based CL dynamics are shown, but the pattern of phases and temporal (non-)locality generalizes to non-equilibrium CL methods. \textbf{A.} Traditional CL using the gradient given by Equation \ref{eq1}. Network states, $a$, at the end of a positive phase (light grey) must be transported through the negative phase (dark grey) to the learning step (red arrow), introducing temporal non-locality. Learning (blue) only occurs after both phases have been completed, and phase lengths, $T$, are fixed. The plotted trajectories are temporal dynamics corresponding to the activation (orange) and network weight (purple) highlighted in the first frame of the diagram. Grey arrows denote dynamics in activation space. \textbf{B.} Equilibrium-based CL using the gradient given by Equation \ref{eq2}, and constant learning rate with stochastic phase lengths. It is temporally local because no network states must be saved in memory to enable learning, and the phase type is chosen randomly at each phase, eliminating periodicity (e.g. two negative phases in a row). There is no highlighted learning step because learning is always on, denoted by the blue weights and apparent in the constantly evolving purple weight dynamics in the bottom plot. Phase lengths are stochastic (e.g. the positive phase is 3x as long as the first negative phase). Data in the two plots is fabricated for the purpose of illustration.}
\label{fig1}
\end{center}
\vskip -0.2in
\end{figure}

We first describe our results on generalizing CL to temporally local learning dynamics, followed by results on learning rate and phase length in equilibrium-based CL systems. We begin by considering the set of CL methods that employ Stochastic Gradient Descent (SGD) to optimize a two-term objective function such that a single SGD step uses the following gradient estimate:
\begin{align}
\grad(\params) = \tp - \tn.
\label{eq1}
\end{align}

Here, $\ps \in \pss$, where $\pss$ is the set of positive-phase network states, and $\ns \in \nss$, where $\nss$ is the set of negative-phase states, $\params$ is the parameters, and $\gp$ and $\gn$ are the functions that map parameters and network states to positive and negative phase gradient terms. Temporal locality and periodicity are inherent in this gradient estimate, as it requires two different sets of network states to compute: separate passes through the neural network must be used to estimate each term of the gradient. We remark that, in a similar vein, batch learning is temporally non-local, as it requires a running sum of network weights to be held in memory during each learning phase. As such, in what follows, we study SGD (batch size of $1$). Many classic CL algorithms seen in the machine learning and computational neuroscience literature have this form of gradient, including contrastive Hebbian learning \cite{movellan1991contrastive}, equilibrium propagation \cite{scellier2017equilibrium}, maximum likelihood learning of energy based models \cite{lecun2006tutorial}, and forms of noise contrastive estimation \cite{gutmann2010noise}. The first contribution of this work is to construct a temporally local estimate of the gradient in Equation \ref{eq1}.

\subsection{Gradient Sampling to Avoid Non-Locality}
We can begin to resolve the temporal non-locality of Equation \ref{eq1} by considering the situation where each term $g_+$ and $g_-$ are used separately to do parameter updates. This means that at the end of a positive phase $g_+$ is used for an update, and $g_-$ is used similarly for the negative phase, avoiding the need for a ``memory buffer" to combine these terms. This might not be a good idea in the case of periodic phases, but it turns out that introducing stochasticity into the phase distribution proves very useful. We therefore propose a probabilistic phase selection process, where the next phase type (positive or negative) to be performed during learning is chosen by a Bernouilli random variable $\prv$. In this setting, the sum in Equation \ref{eq1} can be replaced by the following
\begin{align}
\gest(\params) = \frac{\prv}{\prvp}\tp - \frac{1 - \prv}{1 - \prvp}\tn,
\label{eq2}
\end{align}
where $\prv \sim \mathrm{Bernoulli}(\prvp)$ so that $\prvp$ is the probability of selecting a positive phase for the current SGD step. A realization of such a single, stochastically sampled term becomes a form of importance sampling \cite{robert1999monte} that estimates the true gradient from Equation \ref{eq1}. 
We refer to this gradient estimator as ISD (Importance Sampling-Discrete, given the discrete nature of the Bernoulli variable) which is subject to the following proposition (see Appendix \S \ref{sec:bias-variance} for proof).
\begin{proposition}
For $\prvp \in (0,1)$, the ISD gradient estimate of Equation \ref{eq2} is an unbiased estimate of the gradient given in Equation \ref{eq1}.
\label{prop1}
\end{proposition}

This implies that the mean of $\gest(\params)$ is still the same as that of the original gradient. Furthermore, Equation \eqref{eq2} has the advantage of allowing flexibility in the frequency of each CL phase without sacrificing the unbiased quality. 
Indeed, in the case  where the gradient is that of a scalar-valued loss or energy function, the guaranties of convergence to a fixed point of the loss that SGD relies on still apply \cite{shamir2013stochastic, davis2019stochastic, ghadimi2016mini}. 

Crucially, we note that this gradient estimator remains unbiased for any $\prvp \in (0,1)$. However, not surprisingly, this new gradient will have higher variance in many cases, resulting in noisier estimates. The following lemmas provide some details about this variance (see Appendix \S \ref{sec:bias-variance} for proofs).
\begin{lemma}
Assume $\tp$ and $\tn$ are statistically independent and the inner product of their means is non-negative. If $\prvp \in (0,1)$, the ISD gradient estimate of Equation \ref{eq2} has strictly greater variance (quantified by the trace of the covariance of the gradient vector) than the estimate given by Equation \ref{eq1}.
\label{lem1}
\end{lemma}

Furthermore, the variance, quantified by the trace of the covariance of the gradients, depends non-trivially on the parameter $\prvp$, making the parameter a target for optimization. To this end:

\begin{lemma}
The variance of the gradient estimator given by Equation \ref{eq2} is a convex function of $\prvp$.
\label{lem2}
\end{lemma}
Notably, this minimum can be determined analytically, given estimates of the first and second moments of $\tp$, $\tn$. We explore these properties empirically in numerical experiments outlined in \S \ref{exp-results}, and find that (i) networks can still learn effectively with this higher-variance estimator, and (ii) that analytically derived estimator variance effects outlined in the Lemmas~\ref{lem1} and~\ref{lem2} appear to predict observed variance.

Finally, we highlight that the variance, quantified by the trace of the covariance matrix, of the ISD estimator scales linearly w.r.t. the covariances and the square magnitudes of the means of the gradients of the two gradient terms.

\subsection{Bias of Always-On-Learning and Variable Phase Length in Equilibrium-Based CL}
\label{sec:3.2}

The second contribution of this study is to investigate less restrictive learning dynamics for equilibrium-based CL. We drop the subscripts $+$ or $-$ on the network states, $\ac$, when the equations discussed apply to either phase. Assume that learning occurs via gradient estimate $\gest$. Applying Equation \ref{eq2} to a standard equilibrium-based CL paradigm would require the following learning dynamics:

\begin{equation}
\begin{aligned}
    \prvp_{t+1} &= \mathbbm{1}_{t \neq k\pl} \: \prvp_t + \mathbbm{1}_{t = k\pl} \prv \\
    \data_{t+1} &= \mathbbm{1}_{t \neq k\pl} \: \data_t + \mathbbm{1}_{t = k\pl}X(\prvp_t) \\    
    \ac_{t+1} &\sim \p(\ac_{t+1} | \params_t, \ac_t, \data_{t+1}, \prvp_{t+1}) \\
    \params_{t+1} &= \params_t - \lr(t)\gest(\params_t, \ac_t, \prvp_{t+1}),
    \label{eq4}
\end{aligned}
\end{equation}

where $\p$ is a Markov transition kernel, $\prv$ is sampled according to its role in Equation \ref{eq2}, and $X$ is sampled randomly from the set of positive samples if $\prv=1$ or negative samples if $\prv=0$. To indicate that parameter updates only take place at the end of a phase, we set $\lr(t) = \lr$ if $t = k\pl$ and set $\lr(t) = 0$ otherwise. Here, $\pl$ represents phase length and $\pl$ and $k$ are both in $\mathbb{N}$. While these dynamics are temporally local in the sense that it is no longer necessary to hold distinct positive and negative gradient terms in memory, they still require precise manipulation of $\lr$ and use a fixed phase length. That is to say, we still require the explicit identification of the end of a phase to toggle a parameter update. We now relax these requirements by investigating two novel situations: one where the learning rate is fixed to $\lr(t) = \frac{\lr}{\rn} \forall t$, and another where the learning rate is fixed to the same value \emph{and} $\pl$ is chosen randomly ($\rn$ is the phase length or the mean length respectively). We label these two situations \emph{Always-On-Learning} (AoL) and \emph{AoL Random T} respectively. The learning dynamics of AoL Random T results in Equation \ref{eq4} being updated as follows:

\begin{equation}
\begin{aligned}
    \prvp_{t+1} &= \mathbbm{1}_{\trv_t=0}\prvp_t + \mathbbm{1}_{\trv_t=1}\prv \\
    \data_{t+1} &= \mathbbm{1}_{\trv_t=0}\data_t + \mathbbm{1}_{\trv_t=1}X(\prvp_t) \\
    \ac_{t+1} &\sim \p(\ac_{t+1}|\params_t, \ac_t, \data_{t+1}, \prvp_{t+1}) \\
    \params_{t+1} &= \params_t - \frac{\lr}{\rn}\gest(\params_t, \ac_t, \prvp_{t+1}),
    \label{eq5}
\end{aligned}
\end{equation}

where $\trv_t \sim \mathrm{Bernoulli}(\trvp) ~ \forall t$ and $\trvp << 1$. We remark that this form of continuous updating is only possible with the ISD gradient estimator (Equation \ref{eq2}); the standard estimator (Equation \ref{eq1}) requires both phases to have been run before the next gradient estimate can be calculated.

As with the gradient estimates presented earlier, we would like to show that the learning dynamics with constant learning rate and variable phase length will, at least within some error bound, follow unbiased gradient steps. The following theorem provides a bound on the bias introduced by training with constant learning rate and stochastic phase lengths. We prove and comment on the assumptions of the theorem in \S \ref{proof-of-th1} of the Appendix.

\begin{theorem}
\label{th1}
Consider a single phase of equilibrium-based learning where (i) an update is performed at every time-step of the network dynamics with learning rate $\frac{\lr}{\rn}$, and (ii) the phase is either fixed to value $\rn$ or ends stochastically at each step with small probability $\frac{1}{\rn}$.  Assume that equilibrium dynamics evolve for fixed $\params$ according to the Markov transition kernel $\p(\ac_{t}|\ac_{t-1}, \params)$, and that these equilibrium dynamics weakly converge to a stationary distribution at a rate at least as fast as $\CO(\frac{1}{t^2})$. Then with mild regularity conditions on $\gest(\params)$, and integrability conditions on the transition kernel, the bias of the learning updates are of the following order

\begin{align}
    \CO\Big(\frac{\lr \transient}{\rn}\Big) + \CO(\lr^2 \rn ),
\end{align}

where $\transient = \rn^m$ for $0 < m < 1$.
\end{theorem}

\textbf{Remark:} For asymptotic unbiasedness, and for learning to occur, we need the error to go to zero with a rate faster than the learning rate, $\CO(\lr)$. If we set $\rn = \frac{1}{\sqrt{\lr}} = \transient^2$, then the error rate is $\CO(\lr^{\frac{5}{4}})$, satisfying this constraint.

If one assumes that the equilibrium dynamics follow a deterministic dynamical system, $\ac_{t+1} = F(\ac_t, \params_t)$, instead of stochastic Markov dynamics, one can prove the following theorem, which relies on slightly milder assumptions and yields an improved convergence rate (see Appendix \S \ref{sec:th1d-app} for proof):

\begin{theorem}
\label{th1d}

Assume the same equilibrium learning paradigm of Theorem \ref{th1}, but with dynamics that evolve for fixed $\params$ according to the homogeneous discrete time dynamical system given by the map $F(\ac, \params)$. Assume that $\gest(\params)$ is differentiable and bounded with bounded partial derivatives for all network state values, and that $F$ is differentiable w.r.t. $\ac$ and $\params$, also with bounded partial derivatives. Finally, assume that the equilibrium dynamics converge to a fixed point at a rate at least as fast as $\CO(\frac{1}{t^2})$. Then the bias of the learning updates are of the following order:
\begin{align}
    \CO\Big(\frac{\lr \transient}{\rn}\Big) + \CO(\lr^2),
\end{align}

where $\transient = \rn^m$ for $0 < m < 1$.
\end{theorem}

\begin{figure*}[ht!]
\vskip 0.2in
\begin{center}
\centerline{\includegraphics[width=1.9\columnwidth]{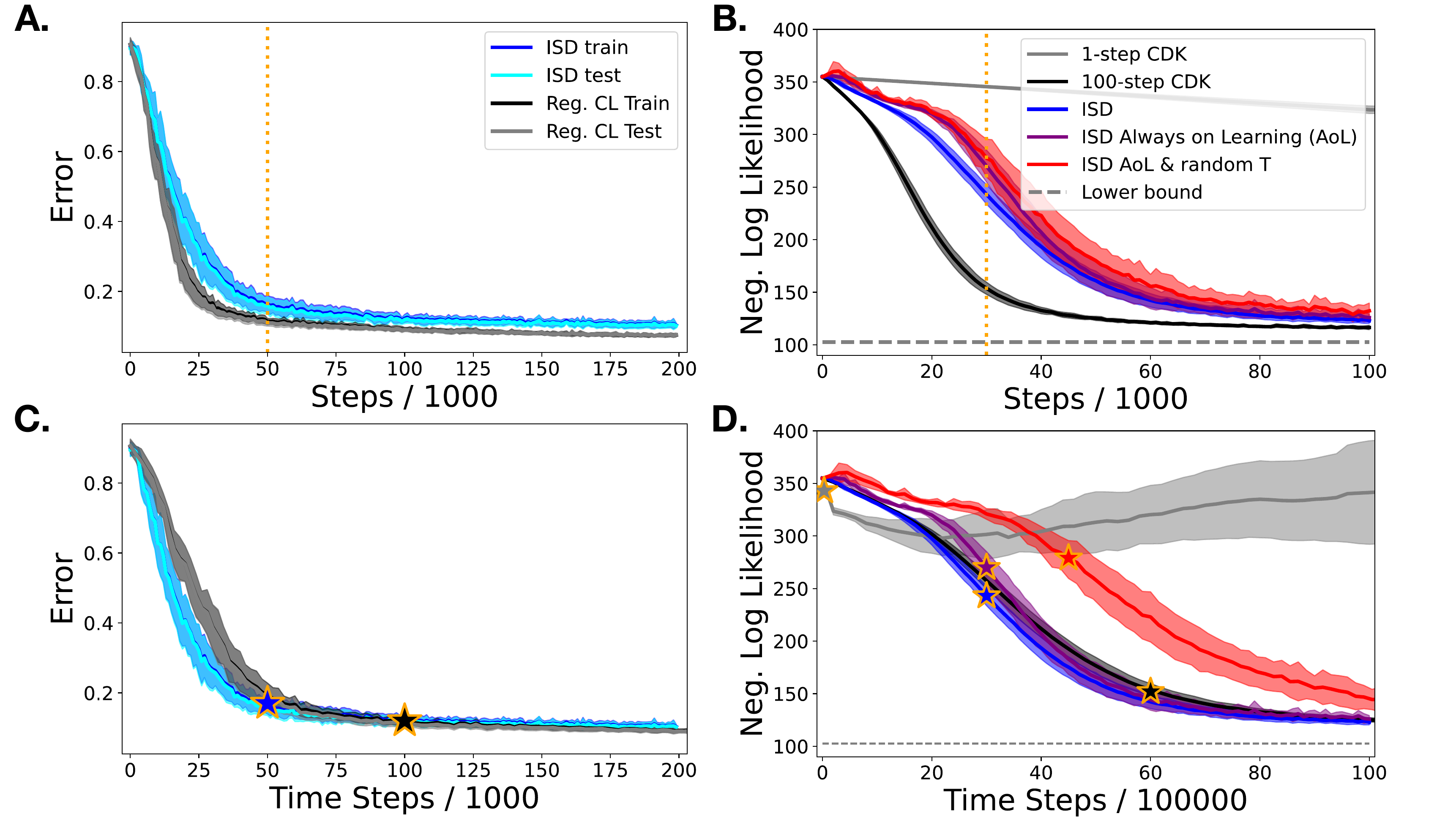}}
\caption{Comparison of ISD, AoL ISD, and AoL random T ISD variants (ours) with standard temporally non-local methods. 
\textbf{A.} CL on MNIST: comparison of test and train data learning curves for the temporally non-local Forward-Forward algorithm versus the ISD version. Y-axis is mean classification error; x-axis is gradient steps divided by 1000. 
\textbf{B.} Equilibrium-based CL on BAS: comparison of 1-step and 100-step CDK algorithm to ISD, AoL ISD, and AoL random $\pl$ ISD for training an RBM on the BAS dataset. Y-axis is negative log-likelihood; x-axis is gradient steps (phase counts) divided by 1000 for CDK and regular ISD (ISD AoL and ISD AoL Random $\pl$).
\textbf{C. \& D.} Same plots as A. and B. but with the x-axis rescaled to show the number of `time' steps used in network dynamics for each algorithm, rather than number of gradient steps or phases (recall that certain algorithms require more or fewer time steps for a single phase). Dotted lines in A. and B. show a step count that corresponds to the number of steps marked by stars in C. and D for each method.
In all cases, ISD methods perform only slightly worse than the temporally non-local, learning rate modulated and fixed phase CDK methods, but with greater variance and some reduction in learning speed over phase counts. This speed distinction shrinks when measuring learning over network dynamics time steps rather than gradient steps (phase counts).
All plotted quantities are means $\pm 1$ standard deviation. For ISD AoL random $\pl$ the x-axis of D is mean time steps.}
\label{res-fig1-3}
\end{center}
\vskip -0.2in
\end{figure*}

Importantly, both Theorems \ref{th1} and \ref{th1d} apply not only to equilibrium-based CL methods, but to \emph{any method} satisfying the theorem assumptions that uses equilibrium-based gradient updates.

Taken together, these results define (mild) conditions for which CL can still converge if parameters are updated at each time step of network dynamics, both with a set phase length and with phases having variable length and switching times. In the next section, we illustrate how these relaxed CL conditions maintain successful learning at a limited cost compared to standard CL. 

\section{Experimental Findings}
\label{exp-results}

While the unbiased (and asymptotically unbiased, in the case of equilibrium CL) property of the proposed algorithms provides theoretical evidence for their utility, we wished to test whether the algorithms would work in practice. We empirically evaluated the proposed algorithms by using them to modify two different CL methods: maximum likelihood learning in energy based models \cite{lecun2006tutorial} and the recently proposed, noise contrastive estimation-inspired \cite{gutmann2010noise} Forward-Forward (FF) algorithm \cite{hinton2022forward}. For the former paradigm we trained a Restricted Boltzmann Machine (RBM) \cite{fischer2014training}, which we selected given its history as a classic model in the computational neuroscience community. For the former, we trained a feed-forward neural network with two hidden layers. We note that the RBM is an equilibrium-based generative model, while the FF algorithm is non-equilbrium-based and is used to train a discriminative model. Testing was performed on the binarized MNIST (bMNIST) and the Bars And Stripes (BAS) \cite{fischer2014training} datasets for the RBM, and MNIST for the FF-trained network. We chose these datasets because the log-likelihood is tractable for BAS, and because both datasets have been used to benchmark related models. Because of the issues of temporal non-locality inherent in batch learning we use batch sizes of $1$ throughout. For simplicity, we tested only vanilla SGD \cite{robbins1951stochastic}, without averaging, momentum, or adaptation, for the RBM; we used ADAM \cite{kingma2014adam} to train the forward-forward network. We stress that we are not comparing performance of FF and RBM; rather, we are using them as two different tests of the bio-plausible algorithms proposed in this paper. Implementation details can be found in  appendix~\ref{exp_details}.

\subsection{Estimators Learn Comparably to Classic Methods}
We first compared ISD with the gradient estimator given by Equation \ref{eq2} to the traditional two-term CL estimator (Equation \ref{eq1}) for training a feed-forward neural network (Fig.\ref{res-fig1-3}.A,C) and a RBM (Fig.\ref{res-fig1-3}.B,D). In all cases, Fig.\ref{res-fig1-3} shows the loss v.s. training time, measured either over gradient steps---or phase counts for the two AoL algorithm versions---(top row) or number of time steps in network dynamics (bottom row). We highlight the importance of contrasting these distinct temporal units, as standard CL updates parameters only at the end of positive/negative phase pairs, each containing several network dynamics iterates (time-steps). In contrast, ISD performs updates more often (either at the end of each phase or at each time step), and the ISD AoL random T version also has a variable phase length. As such, comparing learning curves based on the number of network iterates---the main compute requirement for biological systems---reveals how close ISD methods are to standard CL.

As expected, ISD exhibited higher variance learning curves---a result of using a higher variance gradient estimate---but was still able to learn quite robustly and to a degree of accuracy only slightly worse than standard two-term gradient estimates. In the case of equilibrium learning, the extra bias introduced by the constant learning paradigm led, again, to a further slight reduction in training accuracy and higher noise, which was further accentuated by introducing random phase lengths. While we did not perform an extensive hyper-parameter optimization, we noted that the ISD estimators tended to require a slightly lower learning rate, and converged slightly slower than the two-term estimators.

\begin{figure}[ht!]
\vskip 0.2in
\begin{center}
\centerline{\includegraphics[width=0.95\columnwidth]{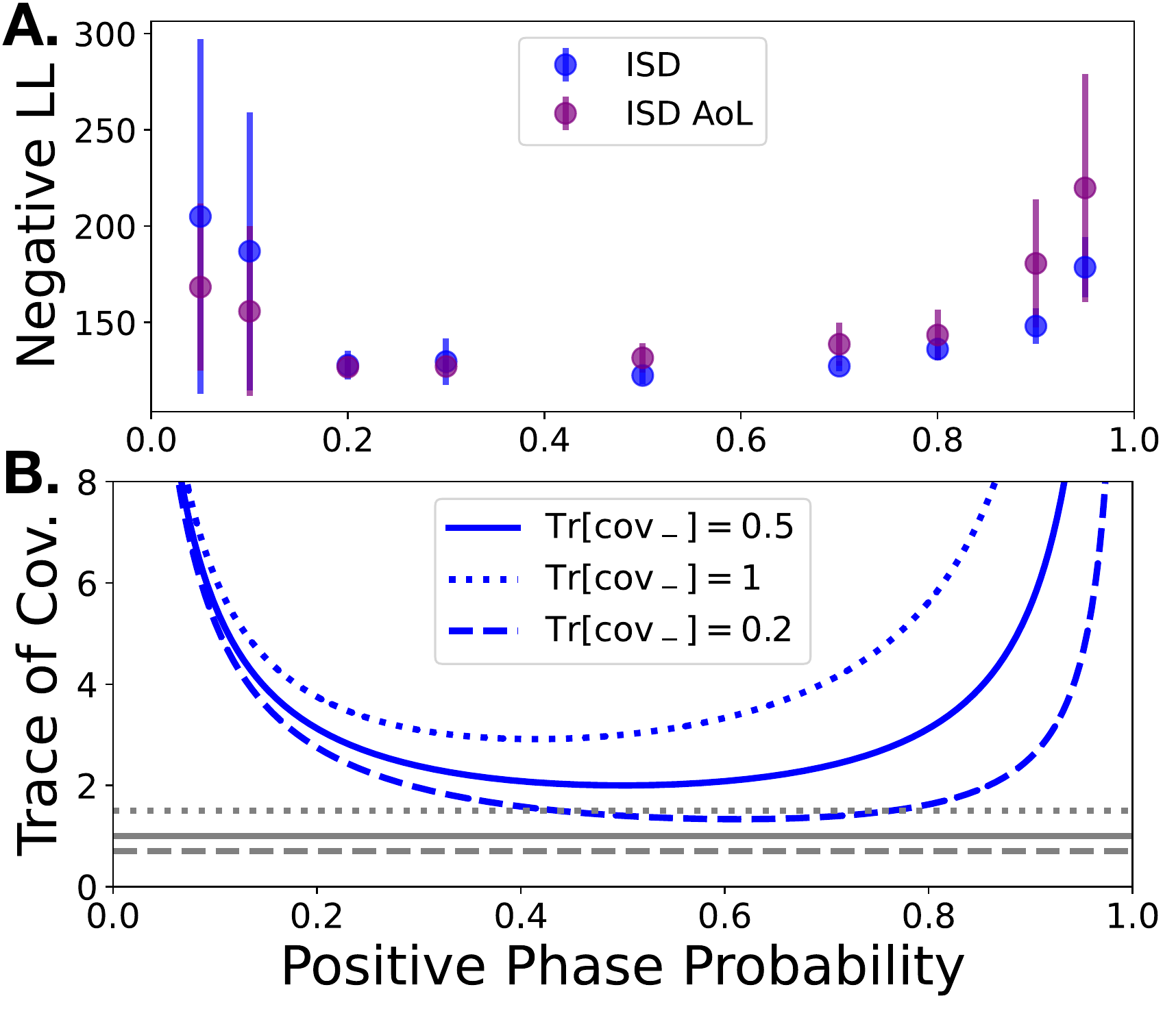}}
\caption{Algorithm performance as function of positive phase probability, $\prvp$. Performance is robust to changes in this parameter, and the optimal value is not always at $\prvp = 0.5$. We expect this relationship to be due to the variance of the estimator. \textbf{A.} Performance on BARS dataset for RBM trained using ISD and ISD with constant learning rate. Error bars are $\pm 1$ standard deviation. \textbf{B.} Gradient estimator covariance matrix trace as function of positive-phase probability, calculated analytically. Different scenarios depict distinct values of the negative phase's gradient estimator covariance trace, holding the positive phase covariance trace constant at $0.5$; e.g. solid line is for $\mathrm{TR}\big[\mathrm{cov}\big( \tp \big)\big] = \mathrm{Tr}\big[\mathrm{cov}\big( \tn \big)\big] = 0.5$. These quantities are a measure of variance of the estimators. Blue lines are the traces of the covariance matrix of the gradient estimate in Equation \ref{eq2}, grey are for that of Equation \ref{eq1}. To simplify the plot we have assumed that the means of $\tp$ and $\tn$ are both the zero vector.}
\label{res-fig2}
\end{center}
\vskip -0.2in
\end{figure}

\subsection{Robustness to Changes in Positive Phase Proportion}
Next, we wished to explore the effect of parameter $\prvp$, the proportion of time spent in positive phases during training, for algorithms making use of the gradient in Equation \ref{eq2}. We investigated this property for the RBM model (Fig.\ref{res-fig2}.A). Interestingly, we found that learning is robust to changes in $\prvp$, and that the optimal value is not always for equal time spent in positive versus negative phases. In fact, we saw better results for a higher $\prvp$ on the bMNIST dataset (see \S \ref{sec:mnist-fig} of the Appendix), and better results for lower values of $\prvp$ on the BAS dataset. Robustness to parameter changes was substantial, with values as high as $0.8$ and as low as $0.2$ not performing drastically worse than the best-performing $\prvp$. 

The mean of the ISD estimator is the same regardless of the ratio of positive to negative phases, so ISD will remain an unbiased estimator as $\prvp$ changes. Conversely, the \emph{variance} of the ISD estimator depends on $\prvp$, a phenomenon that we conjecture to underlie the observed robustness of learning to changes in this parameter. High variance detrimentally affects learning by requiring lower learning rates for an accurate gradient estimate---essentially averaging over more learning steps---and, in absence of this, can result in instability. We investigate ISD estimator variance in the next section.

\subsection{Variance of the Gradient Estimate}

To further interrogate the effect of $\prvp$ on learning, we explored how $\prvp$ affects the variance of the ISD gradient estimator (see Fig.\ref{res-fig2}.B). The gradient estimate is generally a multivariate random variable and thus one uses a covariance matrix to quantify variance. To generalize the variance to higher dimensions we consider the expected Euclidean distance of the random variable from the mean, which is the trace of the covariance matrix and is equal to the variance in one dimension. Given the means and covariances of $\tp$ and $\tn$, one can derive the trace of the covariance matrix of $\gest(\params)$ as a function of $\prvp$ (see \S \ref{sec:bias-variance}). It is this function that is plotted for different assumed values of the means and traces of the covariances of $\tp$ and $\tn$ in Fig.\ref{res-fig2}.B. We observe that higher (lower) values of $\prvp$ lead to less variance when $\tp$ has lower (higher) variance than $ \tn $. We hypothesize that the effect of $\prvp$ on estimator variance may underlie $\prvp$'s effect on learning. Specifically, the variance of the ISD estimator is, for appropriately chosen gradient step-sizes, a convex function of $\prvp$ with rather low curvature around the local minimum (Fig.\ref{res-fig2}.B). Because the variance of the ISD estimator, as a function of $\prvp$, does not increase quickly near the minimum, there is a range of values with sufficiently low variance to allow for robust learning.

Lastly, we know from the theory section that the trace of the covariance of $\gest(\params)$ is a convex function of the positive phase parameter $\prvp \in (0, 1)$, so the value of $\prvp$ leading to minimal estimator variance is unique. If the effect of $\prvp$ on learning is primarily via its effect on the variance of the ISD estimator, the convexity of this function could be used to simplify hyper-parameter initialization. Specifically, one can estimate means and covariances of the positive and negative phase gradient terms and then use our theoretical results (\S \ref{sec:bias-variance} Equation \ref{eq_minimal_alpha}) to solve for the value of $\prvp$ that produces the minimal estimator variance analytically.

\section{Discussion}

This work has implications for experimental and theoretical neuroscience, and neuromorphic computing. For neuroscience, it indicates that a lack of periodicity or sharp learning rate modulation in a neural circuit is not a basis for discounting CL as a potential algorithm implemented by that circuit. Our results add to a body of literature \cite{illing2021local, laborieux2022holomorphic, pogodin2020kernelized} motivating an experimental search for global state variables, e.g. $\prvp$, that should be broadcast to all neurons in a network--perhaps via neuromodulatory signals--to enable learning. In particular, we demonstrate that biological systems implementing CL need not necessarily rely on an oscillatory global signal, as has been previously proposed \cite{baldi1991contrastive, laborieux2022holomorphic}, but could instead implement phase switching through a more stochastic form of gating. This could correspond behaviourally to saccades \cite{illing2021local} or periods of rest associated with hippocampal replay events \cite{davidson2009hippocampal}, for example. Furthermore,  our research exhibits two opportunities for synergy with other recent neuroscience literature. First, it has been hypothesized that neural activation dynamics encode probabilities by sampling \cite{echeveste2020cortical}. Our work provides a potential path to learning bio-plausible probabilistic sampling networks: probabilistic models can be framed as energy based models with a negative log-likelihood energy function \cite{ackley1985learning, lecun2006tutorial}, and the gradient of such a model is of the two-term CL style discussed herein. Second, several recent studies of equilibrium learning provide powerful new learning models, but lack theoretical guarantees for concurrent, online, latent space and learning dynamics \cite{laborieux2022holomorphic, meulemans2022least}. Our study provides guarantees that may be of use, with Theorems \ref{th1} and \ref{th1d}.

For neuromorphic computing, this work provides theoretical and empirical motivation for CL algorithms with more relaxed memory requirements than those of traditional, temporally non-local approaches. Thus, it provides a potential path towards simplifying CL and equilibrium CL algorithms to be more easily implemented on neuromorphic hardware, analogous to past work on Equilibrium Propagation \cite{ernoult2020equilibrium} but for a much broader algorithm class. It is our hope that this will aid in the development of more energy-efficient deep learning.

The lesser memory requirements of the proposed gradient estimator come at the cost of higher estimator variance. In our study this manifested in a need for reduced learning rates, effectively averaging over more gradient steps to combat this variance increase, and further resulted in slight reductions in algorithm performance. This is a clear example of an (at least approximately) unbiased algorithm that trades lower variance for locality. Interestingly, such a trade-off has been noted before for spatial locality \cite{richards2019deep}, as opposed to the temporal locality explored in this study. Further work is warranted to determine whether such a locality-variance trade-off is a general phenomenon for unbiased gradient estimators.

It bears mentioning that the added noise in our learning rule is not a priori effected by the curse of dimensionality, as is the case for certain classic computational neuroscience algorithms \cite{werfel2003learning}. The additional variance introduced by ISD is caused by the variation of a single, scalar Bernoulli variable, $\prvp$; therefore, as shown in Equation \ref{eq_estimator_var}, ISD inherits much of its variance from the behaviour of the positive and negative phase terms of the CL learning rule it is based upon.

This work is, to our knowledge, the first to provide a principled method of manipulating the amount of time spent in the positive versus negative phase for CL, via the importance sampling positive-phase probability $\prvp$. Further work is warranted to determine how this might be generalized to other biphasic algorithms, like Wake-Sleep \cite{dayan1995helmholtz}, and how this parameter might be leveraged for computational benefits. In particular one might be able to trade compute for data in sparse data regimes. In the case of a very low $\prvp$ an algorithm might still perform well using few positive phases---and thus few data points---compared with the equal number of positive and negative phases required by a traditional CL approach. Ultimately, our results provide a compelling example where a gradient estimator with markedly more variance still enables effective learning, while providing important benefits for memory storage and temporal locality.

\section{Conclusion}

In this study we proposed an alternative, temporally local gradient estimate for CL, which can be applied to any CL algorithm that uses a two-term gradient to learn from positive and negative phases. We further proved two theorems applying to algorithms that require equilibrium dynamics to be run to obtain the gradient estimate used at each step of gradient descent. This theorem shows that equilibrium-based algorithms that allow learning to continue during the dynamic phases and, moreover, have stochastic phase lengths, will still exhibit asymptotically unbiased gradients in the small learning rate long mean phase-length limit. Notably, this theorem applies, with some basic assumptions, to all equilibrium-based algorithms, contrastive or otherwise. This theory work was supported by the experiments presented above. In summary, this work suggests that CL does not require the precisely regulated, periodic learning dynamics that are classically used to train this form of AI, thus extending the set of systems capable of implementing CL beyond what might have been previously thought possible.

\section*{Acknowledgements}
The authors acknowledge support from Samsung Electronics Co., Ldt. and the Simons Foundation Collaboration for the Global Brain. GL acknowledges the support from Canada CIFAR AI Chair Program, as well as NSERC Discovery Grant [RGPIN-2018-04821] and the Canada Research Chair in Neural Computations and Interfacing (CIHR, tier 2). EW acknowledges support from three scholarships: a NSERC CGS-D, a FRQNT B2X, and a UNIQUE Excellence Scholarship. Authors wish to thank the other members of the Lajoie lab for their support; in particular, Maximilian Puelma Touzel and Alexandre Payeur for the helpful conversations, and Mohammad Pezeshki for inspiration in the coding of the forward-forward algorithm.


\bibliography{example_paper}

\begin{thebibliography}{47}
\providecommand{\natexlab}[1]{#1}
\providecommand{\url}[1]{\texttt{#1}}
\expandafter\ifx\csname urlstyle\endcsname\relax
  \providecommand{\doi}[1]{doi: #1}\else
  \providecommand{\doi}{doi: \begingroup \urlstyle{rm}\Url}\fi

\bibitem[Ackley et~al.(1985)Ackley, Hinton, and Sejnowski]{ackley1985learning}
Ackley, D.~H., Hinton, G.~E., and Sejnowski, T.~J.
\newblock A learning algorithm for boltzmann machines.
\newblock \emph{Cognitive science}, 9\penalty0 (1):\penalty0 147--169, 1985.

\bibitem[Baldi \& Pineda(1991)Baldi and Pineda]{baldi1991contrastive}
Baldi, P. and Pineda, F.
\newblock Contrastive learning and neural oscillations.
\newblock \emph{Neural computation}, 3\penalty0 (4):\penalty0 526--545, 1991.

\bibitem[Bishop \& Nasrabadi(2006)Bishop and Nasrabadi]{bishop2006pattern}
Bishop, C.~M. and Nasrabadi, N.~M.
\newblock \emph{Pattern recognition and machine learning}, volume~4.
\newblock Springer, 2006.

\bibitem[Cao et~al.(2020)Cao, Summerfield, and Saxe]{cao2020characterizing}
Cao, Y., Summerfield, C., and Saxe, A.
\newblock Characterizing emergent representations in a space of candidate
  learning rules for deep networks.
\newblock \emph{Advances in Neural Information Processing Systems},
  33:\penalty0 8660--8670, 2020.

\bibitem[Chen et~al.(2020)Chen, Kornblith, Norouzi, and Hinton]{chen2020simple}
Chen, T., Kornblith, S., Norouzi, M., and Hinton, G.
\newblock A simple framework for contrastive learning of visual
  representations.
\newblock In \emph{International conference on machine learning}, pp.\
  1597--1607. PMLR, 2020.

\bibitem[Crick \& Mitchison(1983)Crick and Mitchison]{crick1983function}
Crick, F. and Mitchison, G.
\newblock The function of dream sleep.
\newblock \emph{Nature}, 304\penalty0 (5922):\penalty0 111--114, 1983.

\bibitem[Davidson et~al.(2009)Davidson, Kloosterman, and
  Wilson]{davidson2009hippocampal}
Davidson, T.~J., Kloosterman, F., and Wilson, M.~A.
\newblock Hippocampal replay of extended experience.
\newblock \emph{Neuron}, 63\penalty0 (4):\penalty0 497--507, 2009.

\bibitem[Davis \& Drusvyatskiy(2019)Davis and
  Drusvyatskiy]{davis2019stochastic}
Davis, D. and Drusvyatskiy, D.
\newblock Stochastic model-based minimization of weakly convex functions.
\newblock \emph{SIAM Journal on Optimization}, 29\penalty0 (1):\penalty0
  207--239, 2019.

\bibitem[Dayan et~al.(1995)Dayan, Hinton, Neal, and Zemel]{dayan1995helmholtz}
Dayan, P., Hinton, G.~E., Neal, R.~M., and Zemel, R.~S.
\newblock The helmholtz machine.
\newblock \emph{Neural computation}, 7\penalty0 (5):\penalty0 889--904, 1995.

\bibitem[Diederich \& Opper(1987)Diederich and Opper]{diederich1987learning}
Diederich, S. and Opper, M.
\newblock Learning of correlated patterns in spin-glass networks by local
  learning rules.
\newblock \emph{Physical review letters}, 58\penalty0 (9):\penalty0 949, 1987.

\bibitem[Echeveste et~al.(2020)Echeveste, Aitchison, Hennequin, and
  Lengyel]{echeveste2020cortical}
Echeveste, R., Aitchison, L., Hennequin, G., and Lengyel, M.
\newblock Cortical-like dynamics in recurrent circuits optimized for
  sampling-based probabilistic inference.
\newblock \emph{Nature neuroscience}, 23\penalty0 (9):\penalty0 1138--1149,
  2020.

\bibitem[Ernoult et~al.(2020)Ernoult, Grollier, Querlioz, Bengio, and
  Scellier]{ernoult2020equilibrium}
Ernoult, M., Grollier, J., Querlioz, D., Bengio, Y., and Scellier, B.
\newblock Equilibrium propagation with continual weight updates.
\newblock \emph{arXiv preprint arXiv:2005.04168}, 2020.

\bibitem[Fischer \& Igel(2014)Fischer and Igel]{fischer2014training}
Fischer, A. and Igel, C.
\newblock Training restricted boltzmann machines: An introduction.
\newblock \emph{Pattern Recognition}, 47\penalty0 (1):\penalty0 25--39, 2014.

\bibitem[Frenkel et~al.(2021)Frenkel, Bol, and Indiveri]{frenkel2021bottom}
Frenkel, C., Bol, D., and Indiveri, G.
\newblock Bottom-up and top-down neural processing systems design: Neuromorphic
  intelligence as the convergence of natural and artificial intelligence.
\newblock \emph{arXiv preprint arXiv:2106.01288}, 2021.

\bibitem[Friston(2005)]{friston2005theory}
Friston, K.
\newblock A theory of cortical responses.
\newblock \emph{Philosophical transactions of the Royal Society B: Biological
  sciences}, 360\penalty0 (1456):\penalty0 815--836, 2005.

\bibitem[Ghadimi et~al.(2016)Ghadimi, Lan, and Zhang]{ghadimi2016mini}
Ghadimi, S., Lan, G., and Zhang, H.
\newblock Mini-batch stochastic approximation methods for nonconvex stochastic
  composite optimization.
\newblock \emph{Mathematical Programming}, 155\penalty0 (1):\penalty0 267--305,
  2016.

\bibitem[Gutmann \& Hyv{\"a}rinen(2010)Gutmann and
  Hyv{\"a}rinen]{gutmann2010noise}
Gutmann, M. and Hyv{\"a}rinen, A.
\newblock Noise-contrastive estimation: A new estimation principle for
  unnormalized statistical models.
\newblock In \emph{Proceedings of the thirteenth international conference on
  artificial intelligence and statistics}, pp.\  297--304. JMLR Workshop and
  Conference Proceedings, 2010.

\bibitem[Hennequin et~al.(2018)Hennequin, Ahmadian, Rubin, Lengyel, and
  Miller]{hennequin2018dynamical}
Hennequin, G., Ahmadian, Y., Rubin, D.~B., Lengyel, M., and Miller, K.~D.
\newblock The dynamical regime of sensory cortex: stable dynamics around a
  single stimulus-tuned attractor account for patterns of noise variability.
\newblock \emph{Neuron}, 98\penalty0 (4):\penalty0 846--860, 2018.

\bibitem[Hinton(2022)]{hinton2022forward}
Hinton, G.
\newblock The forward-forward algorithm: Some preliminary investigations.
\newblock \emph{arXiv preprint arXiv:2212.13345}, 2022.

\bibitem[Hinton(2002)]{hinton2002training}
Hinton, G.~E.
\newblock Training products of experts by minimizing contrastive divergence.
\newblock \emph{Neural computation}, 14\penalty0 (8):\penalty0 1771--1800,
  2002.

\bibitem[Hinton et~al.(1984)Hinton, Sejnowski, and Ackley]{hinton1984boltzmann}
Hinton, G.~E., Sejnowski, T.~J., and Ackley, D.~H.
\newblock \emph{Boltzmann machines: Constraint satisfaction networks that
  learn}.
\newblock Carnegie-Mellon University, Department of Computer Science
  Pittsburgh, PA, 1984.

\bibitem[Illing et~al.(2021)Illing, Ventura, Bellec, and
  Gerstner]{illing2021local}
Illing, B., Ventura, J., Bellec, G., and Gerstner, W.
\newblock Local plasticity rules can learn deep representations using
  self-supervised contrastive predictions.
\newblock \emph{Advances in Neural Information Processing Systems},
  34:\penalty0 30365--30379, 2021.

\bibitem[Kingma \& Ba(2014)Kingma and Ba]{kingma2014adam}
Kingma, D.~P. and Ba, J.
\newblock Adam: A method for stochastic optimization.
\newblock \emph{arXiv preprint arXiv:1412.6980}, 2014.

\bibitem[Laborieux \& Zenke(2022)Laborieux and Zenke]{laborieux2022holomorphic}
Laborieux, A. and Zenke, F.
\newblock Holomorphic equilibrium propagation computes exact gradients through
  finite size oscillations.
\newblock \emph{arXiv preprint arXiv:2209.00530}, 2022.

\bibitem[Lake et~al.(2017)Lake, Ullman, Tenenbaum, and
  Gershman]{lake2017building}
Lake, B.~M., Ullman, T.~D., Tenenbaum, J.~B., and Gershman, S.~J.
\newblock Building machines that learn and think like people.
\newblock \emph{Behavioral and brain sciences}, 40, 2017.

\bibitem[Lasota \& Mackey(1998)Lasota and Mackey]{lasota1998chaos}
Lasota, A. and Mackey, M.~C.
\newblock \emph{Chaos, fractals, and noise: stochastic aspects of dynamics},
  volume~97.
\newblock Springer Science \& Business Media, 1998.

\bibitem[LeCun et~al.(2006)LeCun, Chopra, Hadsell, Ranzato, and
  Huang]{lecun2006tutorial}
LeCun, Y., Chopra, S., Hadsell, R., Ranzato, M., and Huang, F.
\newblock A tutorial on energy-based learning.
\newblock \emph{Predicting structured data}, 1\penalty0 (0), 2006.

\bibitem[Lillicrap et~al.(2016)Lillicrap, Cownden, Tweed, and
  Akerman]{lillicrap2016random}
Lillicrap, T.~P., Cownden, D., Tweed, D.~B., and Akerman, C.~J.
\newblock Random synaptic feedback weights support error backpropagation for
  deep learning.
\newblock \emph{Nature communications}, 7\penalty0 (1):\penalty0 1--10, 2016.

\bibitem[Meulemans et~al.(2022)Meulemans, Zucchet, Kobayashi, Von~Oswald, and
  Sacramento]{meulemans2022least}
Meulemans, A., Zucchet, N., Kobayashi, S., Von~Oswald, J., and Sacramento, J.
\newblock The least-control principle for local learning at equilibrium.
\newblock \emph{Advances in Neural Information Processing Systems},
  35:\penalty0 33603--33617, 2022.

\bibitem[Movellan(1991)]{movellan1991contrastive}
Movellan, J.~R.
\newblock Contrastive hebbian learning in the continuous hopfield model.
\newblock In \emph{Connectionist models}, pp.\  10--17. Elsevier, 1991.

\bibitem[Oord et~al.(2018)Oord, Li, and Vinyals]{oord2018representation}
Oord, A. v.~d., Li, Y., and Vinyals, O.
\newblock Representation learning with contrastive predictive coding.
\newblock \emph{arXiv preprint arXiv:1807.03748}, 2018.

\bibitem[Ororbia \& Mali(2023)Ororbia and Mali]{ororbia2023predictive}
Ororbia, A. and Mali, A.
\newblock The predictive forward-forward algorithm.
\newblock \emph{arXiv preprint arXiv:2301.01452}, 2023.

\bibitem[Payeur et~al.(2021)Payeur, Guerguiev, Zenke, Richards, and
  Naud]{payeur2021burst}
Payeur, A., Guerguiev, J., Zenke, F., Richards, B.~A., and Naud, R.
\newblock Burst-dependent synaptic plasticity can coordinate learning in
  hierarchical circuits.
\newblock \emph{Nature neuroscience}, 24\penalty0 (7):\penalty0 1010--1019,
  2021.

\bibitem[Pogodin \& Latham(2020)Pogodin and Latham]{pogodin2020kernelized}
Pogodin, R. and Latham, P.
\newblock Kernelized information bottleneck leads to biologically plausible
  3-factor hebbian learning in deep networks.
\newblock \emph{Advances in Neural Information Processing Systems},
  33:\penalty0 7296--7307, 2020.

\bibitem[Rao \& Ballard(1999)Rao and Ballard]{rao1999predictive}
Rao, R.~P. and Ballard, D.~H.
\newblock Predictive coding in the visual cortex: a functional interpretation
  of some extra-classical receptive-field effects.
\newblock \emph{Nature neuroscience}, 2\penalty0 (1):\penalty0 79--87, 1999.

\bibitem[Richards et~al.(2019)Richards, Lillicrap, Beaudoin, Bengio, Bogacz,
  Christensen, Clopath, Costa, de~Berker, Ganguli, et~al.]{richards2019deep}
Richards, B.~A., Lillicrap, T.~P., Beaudoin, P., Bengio, Y., Bogacz, R.,
  Christensen, A., Clopath, C., Costa, R.~P., de~Berker, A., Ganguli, S.,
  et~al.
\newblock A deep learning framework for neuroscience.
\newblock \emph{Nature neuroscience}, 22\penalty0 (11):\penalty0 1761--1770,
  2019.

\bibitem[Robbins \& Monro(1951)Robbins and Monro]{robbins1951stochastic}
Robbins, H. and Monro, S.
\newblock A stochastic approximation method.
\newblock \emph{The annals of mathematical statistics}, pp.\  400--407, 1951.

\bibitem[Robert et~al.(1999)Robert, Casella, and Casella]{robert1999monte}
Robert, C.~P., Casella, G., and Casella, G.
\newblock \emph{Monte Carlo statistical methods}, volume~2.
\newblock Springer, 1999.

\bibitem[Rosenthal(1995)]{rosenthal1995convergence}
Rosenthal, J.~S.
\newblock Convergence rates for markov chains.
\newblock \emph{Siam Review}, 37\penalty0 (3):\penalty0 387--405, 1995.

\bibitem[Scellier \& Bengio(2017)Scellier and Bengio]{scellier2017equilibrium}
Scellier, B. and Bengio, Y.
\newblock Equilibrium propagation: Bridging the gap between energy-based models
  and backpropagation.
\newblock \emph{Frontiers in computational neuroscience}, 11:\penalty0 24,
  2017.

\bibitem[Schuman et~al.(2022)Schuman, Kulkarni, Parsa, Mitchell, Date, and
  Kay]{schuman2022opportunities}
Schuman, C.~D., Kulkarni, S.~R., Parsa, M., Mitchell, J.~P., Date, P., and Kay,
  B.
\newblock Opportunities for neuromorphic computing algorithms and applications.
\newblock \emph{Nature Computational Science}, 2\penalty0 (1):\penalty0 10--19,
  2022.

\bibitem[Shamir \& Zhang(2013)Shamir and Zhang]{shamir2013stochastic}
Shamir, O. and Zhang, T.
\newblock Stochastic gradient descent for non-smooth optimization: Convergence
  results and optimal averaging schemes.
\newblock In \emph{International conference on machine learning}, pp.\  71--79.
  PMLR, 2013.

\bibitem[Sorscher et~al.(2022)Sorscher, Mel, Ocko, Giocomo, and
  Ganguli]{sorscher2022unified}
Sorscher, B., Mel, G.~C., Ocko, S.~A., Giocomo, L.~M., and Ganguli, S.
\newblock A unified theory for the computational and mechanistic origins of
  grid cells.
\newblock \emph{Neuron}, 2022.

\bibitem[Srinivasan et~al.(1982)Srinivasan, Laughlin, and
  Dubs]{srinivasan1982predictive}
Srinivasan, M.~V., Laughlin, S.~B., and Dubs, A.
\newblock Predictive coding: a fresh view of inhibition in the retina.
\newblock \emph{Proceedings of the Royal Society of London. Series B.
  Biological Sciences}, 216\penalty0 (1205):\penalty0 427--459, 1982.

\bibitem[Strubell et~al.(2020)Strubell, Ganesh, and
  McCallum]{strubell2020energy}
Strubell, E., Ganesh, A., and McCallum, A.
\newblock Energy and policy considerations for modern deep learning research.
\newblock In \emph{Proceedings of the AAAI Conference on Artificial
  Intelligence}, volume~34, pp.\  13693--13696, 2020.

\bibitem[Werfel et~al.(2003)Werfel, Xie, and Seung]{werfel2003learning}
Werfel, J., Xie, X., and Seung, H.
\newblock Learning curves for stochastic gradient descent in linear feedforward
  networks.
\newblock \emph{Advances in neural information processing systems}, 16, 2003.

\bibitem[Yin \& Zhang(2005)Yin and Zhang]{yin2005discrete}
Yin, G. and Zhang, Q.
\newblock \emph{Discrete-time Markov chains: two-time-scale methods and
  applications}, volume~55.
\newblock Springer Science \& Business Media, 2005.

\end{thebibliography}
\bibliographystyle{icml2023}

\newpage
\appendix
\onecolumn

\section{Importance Sampling - Discrete (ISD) Gradient Estimate}

\subsection{Bias and Variance of ISD}
\label{sec:bias-variance}

In this section we investigate the bias and variance of the ISD estimator. First, we restate the standard CL gradient estimator and the ISD estimator, along with Proposition \ref{prop1} and Lemmas \ref{lem1} and \ref{lem2} for ease of reference. The standard two term estimate, Equation \ref{eq1} in the main paper, is
\begin{align}
\gest(\params) = \tp - \tn,
\label{eq1-app}
\end{align}

and the ISD gradient estimate, Equation \ref{eq2} in the main paper, is
\begin{align}
\gest(\params) = \frac{\prv}{\prvp}\tp - \frac{1 - \prv}{1 - \prvp}\tn.
\label{eq2-app}
\end{align}

\begin{proposition}
For $\prvp \in (0,1)$, the ISD gradient estimate of Equation \ref{eq2} is an unbiased estimate of the gradient given in Equation \ref{eq1}.
\label{prop1-app}
\end{proposition}

\begin{proof}
Follows trivially by taking expectations w.r.t. $\prv$ and observing that, for a Bernoulli random variable, the expectation is equal to the parameter $\prvp$ so that the $\prvp$ and $1-\prvp$ fraction denominators disappear.
\end{proof}

\textbf{Remark 1:} We further note that this is, by definition, a single-sample importance sampling approach to evaluating the expected value $\mathbb{E}\big(f(\prv)\big)$, this being the two-term sum of Equation \ref{eq1-app}, if $f(1) = \tp$ and $f(0) = \tn$.

\textbf{Remark 2:} In the case where $\prvp = 0.5$, an alternative perspective on this estimate is that it is a version of stochastic gradient descent style uniform sampling from an augmented dataset. This augmented set is composed of all positive and negative samples, each sample augmented with an extra variable indicating whether the sample is positive or negative. This extra variable is then used to determine whether the gradient is added or subtracted to the current parameter value.

\begin{lemma}
Assume $\tp$ and $\tn$ are statistically independent and the inner product of their means is non-negative. If $\prvp \in (0,1)$, the ISD gradient estimate of Equation \ref{eq2} has strictly greater variance (quantified by the trace of the covariance of the gradient vector) than the estimate given by Equation \ref{eq1}.
\label{lem1-app}
\end{lemma}

\begin{lemma}
The variance of the gradient estimator given by Equation \ref{eq2} is a convex function of $\prvp$ on $(0,1)$.
\label{lem2-app}
\end{lemma}

Let us define $\sigp = \mathrm{cov}\big(\tp\big)$, $\sign = \mathrm{cov}\big(\tn\big)$, $\sigpn = \mathrm{cov}\big(\tp, \tn\big)$, $\mup = \E\big(\tp\big)$, and $\mun = \E\big(\tn\big)$ To prove the first Lemma, we first note that the covariance matrix of the standard CL gradient estimate given by Equation \ref{eq1-app} is $\sigp + \sign - 2\sigpn$. We now prove the following proposition, which provides the bulk of the proof of the two lemmas:

\begin{proposition}
Assume $\prvp \in (0, 1)$. The covariance matrix for the ISD gradient estimate vector defined by Equation \ref{eq2-app} is
\begin{align}
    \Sigma = \frac{1}{\prvp}\sigp + \frac{1}{1-\prvp}\sign + \frac{1-\prvp}{\prvp}\mup\mup^\top + \frac{\prvp}{1-\prvp}\mun\mun^\top + \mup\mun^\top + \mun\mup^\top.
    \label{eq_estimator_var}
\end{align}

Moreover, the trace of this matrix is a convex function of $\prvp$, with minimum at
\begin{align}
    \prvp_\mathrm{min} = \frac{\mathrm{Tr}[\sigp + \mup\mup^\top] - \sqrt{\mathrm{Tr}[\sigp + \mup\mup^\top]\mathrm{Tr}[\sign + \mun\mun^\top]}}{\mathrm{Tr}[\sigp + \mup\mup^\top]-\mathrm{Tr}[\sign + \mun\mun^\top]},
    \label{eq_minimal_alpha}
\end{align}
\end{proposition}

\begin{proof}
We wish to calculate the covariance of the gradient estimate, as a function of $\prvp$, to compare with the covariance of the original gradient estimate, and to see when the trace of this covariance is minimized w.r.t. $\prvp$. Note that we use the trace of the covariance instead of the variance since we are working in a multi-dimensional space. We begin by deriving the matrix of second moments of the gradient vector, which we denote by $\mathrm{corr}$. We also define $\fp = \frac{\prv}{\prvp}\tp$ and $\fn = -\frac{1 - \mathcal{\prv}}{1 - \prvp}\tn$ to simplify notation. Now, note that under this notation, the $i^{th}$ element of the ISD gradient vector estimate is $\fp_i + \fn_i$. Thus
\begin{align}
    &\mathrm{corr}\big(\fp_i + \fn_i, \fp_j + \fn_j \big) \nonumber \\ &= \mathbb{E}\bigg[\fp_i\fp_j + \fp_i\fn_j + \fp_j\fn_i + \fn_i\fn_j\bigg] \nonumber \\
    &= \mathbb{E}\bigg[\fp_i\fp_j + \fn_i\fn_j\bigg] \nonumber
     \\
    &= \frac{1}{\prvp}[\Sigma _{+ij} + \mu_{+i}\mu_{+j}] + \frac{1}{1-\prvp}[\Sigma_{-ij} + \mu_{-i}\mu_{-j}],
\end{align}

where in the second line we used linearity of expectation and the fact that $\prv(1 - \prv) = 0$, and the in third line we used the form of the second moments of a Bernoulli random variable and the relationship between covariance and the second moment. Subtracting the means to arrive at the covariance and rearranging, we get Equation \ref{eq_estimator_var}.

Observe that this function has vertical asymptotes at $\prvp = 0$ and $\prvp = 1$. We are interested in finding the minimum of the trace of the matrix w.r.t. $\prvp$ on the interval $\alpha\in (0, 1)$. Differentiating gives
\begin{align}
    \frac{\partial\mathrm{Tr}[\Sigma]}{\partial\prvp} &= -\frac{1}{\prvp^2}\mathrm{Tr}[\sigp] +\frac{1}{(1-\prvp)^2}\mathrm{Tr}[\sign] - \frac{1}{\prvp^2}\mathrm{Tr}[\mup\mup^\top] + \frac{1}{(1-\prvp)^2}\mathrm{Tr}[\mun\mun^\top] \nonumber \\
    &= -\frac{1}{\prvp^2}\mathrm{Tr}[\sigp + \mup\mup^\top] + \frac{1}{(1-\prvp)^2}\mathrm{Tr}[\sign + \mun\mun^\top]
\end{align}

Setting this equal to zero and solving for $\prvp$ using the quadratic formula gives:
\begin{align}
    \prvp = \frac{\mathrm{Tr}[\sigp + \mup\mup^\top] \pm \sqrt{\mathrm{Tr}[\sigp + \mup\mup^\top]\mathrm{Tr}[\sign + \mun\mun^\top]}}{\mathrm{Tr}[\sigp + \mup\mup^\top]-\mathrm{Tr}[\sign + \mun\mun^\top]}.
\end{align}

Some simple algebra shows that only the difference and not the sum in the above equation yields a critical point on the interval $(0, 1)$, so the difference is the minimum we desire. Finally, differentiating one more time gives:
\begin{align}
    \frac{\partial^2\mathrm{Tr}[\Sigma]}{\partial\prvp^2} &= \frac{2}{\prvp^3}\mathrm{Tr}[\sigp + \mup\mup^\top] + \frac{2}{(1-\prvp)^3}\mathrm{Tr}[\sign + \mun\mun^\top].
\end{align}

The trace of the sum of a covariance matrix and a vector outer product is positive, so this expression is positive on $(0, 1)$, thus we get a convex function of $\prvp$ on this interval.
\end{proof}

Lemma \ref{lem2-app} is proven by the proof of the above proposition. For Lemma \ref{lem1-app}, it suffices to note that if $\tp$ is statistically independent of $\tn$ then their covariance is given by $\sigp + \sign$. The proof of Lemma \ref{lem1-app} follows by the linearity of the trace and the assumption that $\mathrm{Tr}(\mup\mun^\top) = \mathrm{Tr}(\mun\mup^\top) = \mup^\top\mun \geq 0$.

\section{Bias of Constant Learning Rate for Equilibrium Learning Gradient Estimates}
\label{proof-of-th1}

We investigate an equilibrium learning system where, to estimate the gradient needed for one parameter update, $\gest(\params)$, one must run internal dynamics, following a Markov chain, on some activation space $\as$ until these dynamics are sufficiently close to an equilibrium.

We investigate the question of how much bias is introduced into the SGD dynamics of this equilibrium-learning system when, instead of fixing $\params$ during each phase of dynamics before each gradient step, one allows gradient steps to be taken at \emph{every step of the dynamics}. Specifically, we are interested in proving Theorem \ref{th1}.

In an abuse of notation we allow $\p$ to represent any probability density (or mass function) appearing in the proof, with each distribution being discerned by the arguments---e.g. $\p(\ac_t | \ac_0, \params) = \int_{\sz \times \dots \times \sz}\prod_{k=1}^{t}\p(\ac_k | \ac_{k-1}, \params) \dee \ac_1 \dots \dee \ac_{t-1}$ is a $t$ step transition kernel and $p(\ac_k|\ac_{k-1}, \params)$ is a single step transition. We adopt a discrete time framework with $t \in \mathbb{N}$ given that SGD is typically implemented in discrete time, though we expect similar results for continuous time. The proposed algorithm views each stimulus, or samples from its internal model, for $\pl$ time-steps and the learning rate at each time-step is scaled by a factor of $\frac{1}{\rn}$, which is the mean of $\pl$. We use $\params_t \in \prs \: \forall \: t$ to denote the parameter value at time $t$, and $\ac_t$ to denote the network state at this time. $\lr$ denotes the learning rate before scaling by the inverse of $\rn$.

\begin{theorem}
\label{th1-app}

Consider a single phase of equilibrium-based learning where the phase is either fixed to value $\rn$ or ends stochastically at each step with small probability $\frac{1}{\rn}$ and, instead of waiting until the end of the phase to do an update with learning rate $\lr$, an update is performed at every step of the equilibrium dynamics with learning rate $\frac{\lr}{\rn}$. Assume that equilibrium dynamics evolve for fixed $\params$ according to the Markov transition kernel $\p(\ac_{t}|\ac_{t-1}, \params)$. Further assume that $\gest(\params)$ is differentiable and bounded, with a bounded derivative, for all network state values, the Markov kernel for the equilibrium dynamics is continuously differentiable w.r.t. $\params$ for all network state values, the $k$-step kernel $\p(\ac_t|\ac_{t-k}, \params)$ is bounded w.r.t. $\ac_{t-1}$ for all parameter values, and the function $\extra(x)=\int_\sz \Big|\frac{\partial \p(x | y, \xi_\mathrm{max})}{\partial\params}\Big| \p(y | \ac_0, \params_0) \dee y$ is integrable for the given initial conditions and all values of $\xi_\mathrm{max}$. Finally, assume that the equilibrium dynamics converge to a stationary distribution weakly with a rate as least as fast as $\CO(\frac{1}{t^2})$. Then the bias of the learning updates are of the following order:
\begin{align}
    \CO\Big(\frac{\lr \transient}{\rn}\Big) + \CO(\lr^2 \rn ),
\end{align}

where $\transient = \rn^m$ for $0 < m < 1$ and $\CO$ is big-O notation.
\end{theorem}

\textbf{Remark on dimension of parameters:} The following proof relies very heavily on the multivariate corollary of the mean value theorem, which we state in section \ref{sec:mvt} for completeness.

\subsection{Proof of Theorem \ref{th1}}

\begin{proof}

In what follows we restrict $t$ to the interval $[0, \pl)$---that is, one period of learning. $\pl$ can either be chosen deterministically or stochastically, as we will discuss below. The full dynamics of network state and parameter updates are given by the following equations:

\begin{equation}
\begin{aligned}
    \ac_{t+1} &\sim \p(\ac_{t+1}|\ac_t, \params_t) \\
    \params_{t+1} &= \params_t - \frac{\lr}{\rn}\gest(\ac_t, \params_t),
    \label{eq_dynamics}
\end{aligned}
\end{equation}

where $\gest$ is the gradient estimate for the model, $\p(\ac_{t+1}|\ac_t, \params)$ is, for fixed $\params$, the transition function of a homogeneous Markov chain that we assume converges uniquely to the stationary distribution of the model. Thus, if we set $\params_t=\params$ for all $t$, $\{\ac_t\}_{t=0}^\pl$ would be a Markov chain performing MCMC sampling from our probability model with parameter $\params$. However, given the dependence on $\params_t$, $\{\ac_t\}_{t=0}^\pl$ is not performing exact MCMC sampling and is, in fact, not even a Markov chain, due to the sequence of temporal dependencies induced by $\params_t$. Lastly, we note that, conditioned on $\ac_t$, $\params_{t+1}$ is a deterministic variable since the input stimulus is not changing over the period $[0, \pl)$ and the only source of randomness is via $\ac$.

Our goal is to prove that the sum of parameter changes over the phase $t \in [0, \pl)$, which we denote by the random variable $\gs(\ac_0, \params_0)$, is asymptotically unbiased in the limit of small learning rate, $\lr$, and long mean stimulus presentation time, $\rn$. That is, we will prove the following:

\begin{align}
    \mathbb{E}\big[\gs(\ac_0, \params)\big] = \lr\gest(\params) + \CO\Big(\frac{\lr \transient}{\rn}\Big) + \CO(\lr^2 \rn).
    \label{eq_main_result}
\end{align}

To this end, we simply linearize about $\params_0$ so as to remove the temporal dependence brought about by the continued learning. This allows us to make use of the convergence properties of the homogeneous Markov chain. We now proceed in three steps: (1) expand $\gest$ to remove the dependence of the gradient, via $\params_t$, on the slow learning timescale; (2) expand $\p(\ac_t|\ac_0, \params_0)$ to remove the rest of the expectation's dependence---via the system's internal dynamics---on the slow timescale; (3) leverage the convergence of the homogeneous Markov chain to the stationary distribution to obtain samples to approximate the desired expectations. Steps (1) and (2) jointly lead to the error term that is quadratic in $\lr$, in Equation \ref{eq_main_result}, and step (3) results in the other error term. Before embarking on the three main steps of the proof, we must first engage in a preliminary analysis to account for stochastic phase lengths.

\textbf{Preliminary for Stochastic Phase Length:} Using the law of total expectation (tower property), and defining the random variable $\pls = \mathbbm{1}_{\pl \geq \transient}$

\begin{equation}
\begin{aligned}
    \E\big[\gs(\ac_0, \params)\big] &= \E\big[\E[\gs(\ac_0, \params)|\pls]\big] = \E[\gs(\ac_0, \params)|\pls = 1]\p(\pls = 1) + \E[\gs(\ac_0, \params)|\pls = 0]\p(\pls = 0) \\
    &= \E[\gs(\ac_0, \params)|\pls = 1] + \big(\E[\gs(\ac_0, \params)|\pls = 0] - \E[\gs(\ac_0, \params)|\pls = 1]\big)\p(\pls = 0).
\end{aligned}
\end{equation}

Observe that, by the assumption that the gradient updates are bounded, $\gs(\ac_0, \params_0) \leq K_1\lr\pl/\rn$ for some constant $K_1$. Furthermore, it can be shown that $\p(\pls = 0) = \CO(\transient / \rn)$ (as we demonstrate in \S \ref{sec:prob-app}). Thus

\begin{equation}
\begin{aligned}
    \E\big[\gs(\ac_0, \params)\big] &= \E[\gs(\ac_0, \params)|\pls = 1] + \CO\Big(\frac{\lr \transient \E(\pl) }{\rn^2}\Big) = \E[\gs(\ac_0, \params)|\pls = 1] + \CO\Big(\frac{\lr \transient}{\rn}\Big),
    \label{eq-prob-analysis}
\end{aligned}
\end{equation}

where we have used $\E(\pl) = \rn$ (see \S \ref{sec:moments-of-pl}). 

What Equation \ref{eq-prob-analysis} means is that, for a large enough average phase length, we can assume that we are dealing with a value of $\pl$ that is larger than some deterministic value, $\transient$, even in the case of stochastic phase lengths--albeit with the above error term. Moreover, if we apply the tower property again, we get:

\begin{align}
    \E[\gs(\ac_0, \params)|\pls = 1] = \E_\pl\big[\E[\gs(\ac_0, \params)|\pl]|\pls = 1\big].
\end{align}

For the three main steps of the proof we will thus analyze the inner expectation, $\E[\gs(\ac_0, \params)|\pl]$ (in an abuse of notation we will refrain from explicitly writing the condition that $\pl \geq \transient$), the analysis of which is equivalent to the case of a deterministic phase length. At the end of the proof we will take into account the stochasticity in $\pl$. We now proceed to the three main steps of the proof.

\textbf{Step 1:} This section is a straightforward application of Lemma \ref{mvt-corr}. Observe that one can write the parameter dynamics as:
\begin{align}
    \params_t = \params_0 + \frac{\lr}{\rn}\sum_{k=0}^{t-1}\gest(\ac_k, \params_k),
\end{align}

and thus,
\begin{align}
    \gs(\ac_0, \params_0) = \frac{\lr}{\rn}\sum_{t=0}^{\pl-1}\gest(\ac_t, \params_t).
\end{align}

Let us focus on the $\vi^{th}$ element of $\gest$. Assuming that $\gest \in \CC^0(\params)$ (differentiable w.r.t. $\params$), we can expand it to first order about $\params_0$:
\begin{align}
    \gest_\vi(\ac_t, \params_t) = \gest_\vi\Big(\ac_t, \params_0 + \frac{\lr}{\rn}\sum_{k=0}^{t-1}\gest(\ac_k, \params_k)\Big) = \gest_\vi(\ac_t, \params_0) + \CO\Big(\frac{\lr}{\rn}\Big|\Big|\sum_{k=0}^{t-1}\gest(\ac_k, \params_k)\Big|\Big|_1\Big).
    \label{eq-20}
\end{align}

If we assume that the gradient is bounded (e.g. we clip the gradient above some sufficiently large value), and that the partial derivatives of the gradient w.r.t. $\params$ are bounded, then the second term on the RHS is $\CO\big(t\frac{\lr}{\rn}\big)$. Thus,
\begin{align}
    \gs(\ac_0, \params_0) &= \frac{\lr}{\rn}\sum_{t=0}^{\pl-1}\gest(\ac_t, \params_0) + \CO\Big(\frac{\lr^2}{\rn^2}\Big)\sum_{t=0}^{\pl-1}t.
\end{align}

The last factor in the final sum is order $\pl^2$ by the quadratic nature of the triangular summation.

\textbf{Step 2:} By linearity of expectation, and using the results of step 1,
\begin{align}
    \mathbb{E}\big[\gs(\ac_0, \params_0)|T \big] = \frac{\lr}{\rn}\sum_{t=0}^{\pl-1}\mathbb{E}\big[\gest(\ac_t, \params_0)\big] + \CO\Big(\lr^2\frac{T^2}{\rn^2}\Big).
    \label{S2_main}
\end{align}

Focusing on the $t^{\mathrm{th}}$ expectation on the RHS:
\begin{align}
    \mathbb{E}\big[\gest(\ac_t, \params_0)|T \big] = \int_{\sz} \gest(\ac_t, \params_0) p(\ac_t | \ac_0, \params_0) \dee \ac_t,     \label{eq_S2_prob_intermediate_1}
\end{align}

Focusing specifically on the probability, we get:

\begin{equation}
\begin{aligned}
    \p(\ac_t | \ac_0, \params_0) &= \int_{\sz}\int_{\prs} \p(\ac_t, \ac_{t-1}, \params_{t-1} | \ac_0, \params_0)\dee \params_{t-1} \dee \ac_{t-1} \\
    &= \int_{\sz}\int_{\prs} \p(\ac_t | \ac_{t-1}, \params_{t-1}, \ac_0, \params_0) \p(\ac_{t-1}, \params_{t-1} | \ac_0, \params_0) \dee \params_{t-1} \dee \ac_{t-1} \\
    &= \int_{\sz}\int_{\prs} \p(\ac_t | \ac_{t-1}, \params_{t-1}) \p(\ac_{t-1}, \params_{t-1} | \ac_0, \params_0) \dee \params_{t-1} \dee \ac_{t-1},
\end{aligned}
\end{equation}

where, in the final line, we used the dependence structure of the MCMC dynamics. 

Let us define $\Delta \params_t = \frac{\lr}{\rn}\sum_{k=0}^{t-1}\gest(\ac_k, \params_k)$ and use again our assumption that the gradient is bounded, so that in $t$ steps $||\Delta\params_t||_1 \leq \Delta_{\mathrm{max}}\params_t$, a deterministic quantity. Assume further that the transition function is differentiable w.r.t. $\params$ (in most cases this follows from our assumption that $\gest$ is differentiable). Then, the mean value theorem (see Lemma \ref{mvt-corr}), we have:

\begin{align}
    &= \int_{\sz}\int_{\prs} \bigg[\p(\ac_t | \ac_{t-1}, \params_0) + \Delta\params_{t-1} \cdot \frac{\partial \p(\ac_t | \ac_{t-1}, \xi_{t-1})}{\partial\params}\bigg] \p(z_{t-1}, \params_{t-1} | \ac_0, \params_0) \dee \params_{t-1} \dee \ac_{t-1}
    \label{eq-25}
\end{align}

where $\xi_{t-1}$ is in a cube of edge length $||\Delta\params_{t-1}||_1$ centered at $\theta_0$, and is a random variable on account of the variability in $\Delta\params_t$. We now proceed as follows

\begin{equation}
\begin{aligned}
    &= \int_{\sz}\int_{\prs} \p(\ac_t | \ac_{t-1}, \params_0) \p(\ac_{t-1}, \params_{t-1} | \ac_0, \params_0)\dee \params_{t-1} \dee \ac_{t-1} \\ &+ \int_{\sz}\int_{\prs} \Delta\params_{t-1} \cdot \frac{\partial \p(\ac_t | \ac_{t-1}, \xi_{t-1})}{\partial\params} \p(\ac_{t-1}, \params_{t-1} | \ac_0, \params_0) \dee \params_{t-1} \dee \ac_{t-1} \\
    &= \int_{\sz}\p(\ac_t | \ac_{t-1}, \params_0) \p(\ac_{t-1} | \ac_0, \params_0) \dee \ac_{t-1}     \label{eq_S2_prob_intermediate_2} \\ &+ \int_{\sz}\int_{\prs} \Delta\params_{t-1} \cdot \frac{\partial \p(\ac_t | \ac_{t-1}, \xi_{t-1})}{\partial\params} \p(\ac_{t-1}, \params_{t-1} | \ac_0, \params_0) \dee \params_{t-1} \dee \ac_{t-1}
\end{aligned}
\end{equation}

The second term on the RHS of the final line is a bias term, which we will now bound. Observe:

\begin{equation}
\begin{aligned}
    & \Bigg| \int_{\sz}\int_{\prs} \Delta\params_{t-1} \cdot \frac{\partial \p(\ac_t | \ac_{t-1}, \xi_t)}{\partial\params} \p(\ac_{t-1}, \params_{t-1} | \ac_0, \params_0) \dee \params_{t-1} \dee \ac_{t-1}\Bigg| \\
    & \leq \int_{\sz}\int_{\prs} \Delta_{\mathrm{max}}\params_{t-1}\Big|\Big|\frac{\partial \p(\ac_t | \ac_{t-1}, \xi_{t-1})}{\partial\params}\Big|\Big|_1 \p(\ac_{t-1}, \params_{t-1} | \ac_0, \params_0) \dee \params_{t-1} \dee \ac_{t-1},
\end{aligned}
\end{equation}

where the inequality follows from Jensen's inequality, to bring the absolute values inside the integral, and the use of $\Delta_\mathrm{max}\params_{t-1}$. Integrating out the parameters at time $t-1$ and applying $\Delta_{\mathrm{max}}\params_{t-1} \leq \Delta_{\mathrm{max}}\params_{\pl}$,

\begin{equation}
\begin{aligned} 
    &= \Delta_{\mathrm{max}}\params_{\pl} \int_{\sz} \Big|\Big|\frac{\partial \p(\ac_t | \ac_{t-1}, \xi_\mathrm{max})}{\partial\params}\Big|\Big|_1 \p(\ac_{t-1} | \ac_0, \params_0) \dee \ac_{t-1} \\ 
    &= \CO\Big(\lr\frac{T}{\rn}\Big)\extra(\ac_t),
    \label{eq_S2_prob_intermediate_3}
\end{aligned}
\end{equation}

where we have used the fact that $\CO(\Delta_{\mathrm{max}}\params_{\pl}) = \CO\big(\lr\frac{\pl}{\rn}\big)$. We have further assumed that the transition function is $\CC^1(\params)$, so that by continuity on the compact cube of edge length $\Delta_{\mathrm{max}}\params_{\pl}$ centered at $\theta_0$ we get $\Big|\Big|\frac{\partial \p(\ac_t | \ac_{t-1}, \xi_{t-1})}{\partial\params}\Big|\Big|_1 \leq \Big|\Big|\frac{\partial \p(\ac_t | \ac_{t-1}, \xi_{\mathrm{max}})}{\partial\params}\Big|\Big|_1$ for some value of $\xi_\mathrm{max}$. Lastly, we have simplified the expression with the function $\extra(x) = \int_\mathcal{Y} \Big|\Big|\frac{\partial \p(x | y, \xi_\mathrm{max})}{\partial\params}\Big|\Big|_1 \p(y | \ac_0, \params_0) \dee y$ defined in the theorem statement. We now have:
\begin{align}
    \p(\ac_t|\ac_0,\params_0) = \int_{\sz}\p(\ac_t | \ac_{t-1}, \params_0) \p(\ac_{t-1} | \ac_0, \params_0) \dee \ac_{t-1} + \CO\Big(\lr\frac{\pl}{\rn}\Big)\extra(\ac_t).
        \label{eq_S2_prob_intermediate_4}
\end{align}

If $t=2$ we stop here. For $t>2$, observe that the second factor in the integrand is simply the original transition function moved one time-step back in time, and that the steps we have gone through to arrive at this point work for an arbitrary time-step. That is, for $k \in \{2,\dots,\pl\}$, we have:
\begin{align}
    \p(\ac_k | \ac_0,\params_0) = \int_{\sz}\p(\ac_{k} | \ac_{k-1}, \params_0) \p(\ac_{k-1} | \ac_0, \params_0) \dee \ac_{k-1} + \CO\Big(\lr\frac{\pl}{\rn}\Big)\extra(\ac_k),
    \label{eq_s2_recursive}
\end{align}

This suggests taking a recursive approach to expanding $\p(\ac_t|\ac_0,\params_0)$, which is what we do now. Applying Equation \ref{eq_s2_recursive} to expand Equation \ref{eq_S2_prob_intermediate_4}, and writing $\p(\ac_t|\ac_{t-1}, \params_0) = \p(\ac_t|\ac_{t-1})$, we get:

\begin{equation}
\begin{aligned} 
    \p(\ac_t|\ac_0, \params_0) &= \int_{\sz \times \sz}\p(\ac_t | \ac_{t-1})\p(\ac_{t-1} | \ac_{t-2}) \p(\ac_{t-2} | \ac_0) \dee \ac_{t-2} \times \dee \ac_{t-1} \\
    &+ \int_{\sz}\p(\ac_t|\ac_{t-1})\CO\Big(\lr\frac{T}{\rn}\Big)\extra(\ac_{t-1})\dee \ac_{t-1} + \CO\Big(\lr\frac{\pl}{\rn}\Big)\extra(\ac_t).
\end{aligned}
\end{equation}

Applying the recursive Equation \ref{eq_s2_recursive} $t-3$ more times yields the following

\begin{equation}
\begin{aligned} 
    &= \int_{\sz} \dots \int_{\sz} \prod_{k=0}^{t-1}\p(\ac_{t-k} | \ac_{t-k-1}, \params_0) \dee \ac_{1} \dots \dee \ac_{t-1} \\
    &+\sum_{k=1}^{t-1} \int_{\sz}\p(\ac_t|\ac_{t-k})\CO\Big(\lr\frac{\pl}{\rn}\Big)\extra(\ac_{t-k})\dee \ac_{t-k} + \CO\Big(\lr\frac{\pl}{\rn}\Big)\extra(\ac_t).
    \label{eq_S2_fully_expanded}
\end{aligned}
\end{equation}

We now insert Equation \ref{eq_S2_fully_expanded} into Equation \ref{eq_S2_prob_intermediate_1}:
\begin{align}
    \mathbb{E}\big[\gest(\ac_t, \params_0)|T \big] &= \int_{\sz} \gest(\ac_t, \params_0) \int_{\sz} \dots \int_{\sz} \prod_{k=0}^{t-1}\p(\ac_{t-k} | \ac_{t-k-1}, \params_0) \dee \ac_{1} \dots \dee \ac_t \nonumber \\
    &+ \int_{\sz} \gest(\ac_t, \params_0) \sum_{k=1}^{t-1} \int_{\sz}\p(\ac_t|\ac_{t-k})\CO\Big(\lr\frac{\pl}{\rn}\Big)\extra(\ac_{t-k})\dee \ac_{t-k} \dee \ac_t \nonumber \\
    &+ \int_{\sz} \gest(\ac_t, \params_0) \CO\Big(\lr\frac{\pl}{\rn}\Big)\extra(\ac_t) \dee \ac_t \nonumber \\
    &= \int_{\sz} \gest(\ac_t, \params_0) \int_{\sz} \dots \int_{\sz} \prod_{k=0}^{t-1}\p(\ac_{t-k} | \ac_{t-k-1}, \params_0) \dee \ac_{1} \dots \dee \ac_t + \CO\Big(\lr \frac{T^2}{\rn}\Big),
\end{align}

where for the last equality we used our previous assumption that $\gest(\ac_t, \params_0)$ is bounded, along with the assumption that $\p(\ac_t|\ac_{t-k})$ is bounded w.r.t. $\ac_{t-k}$ and that $\extra(\ac_{t-k})$ is integrable, for $k \in \{0,\dots,t-1\}$.

Now inserting this in Equation \ref{S2_main} we get:
\begin{align}
    \mathbb{E}\big[\gs(\ac_0, \params_0)|T \big] &= \frac{\lr}{\rn}\sum_{t=0}^{\pl-1}\int_{\sz} \gest(\ac_t, \params_0) \p_{t}(\ac_t)\dee \ac_t + \CO\Big(\lr^2 \frac{\pl^3}{\rn^2}\Big) \nonumber \\
    &=\frac{\lr}{\rn}\sum_{t=0}^{\pl-1}\mathbb{E}_{\ac\sim \p_t}\big(\gest(\ac,\params)\big) + \CO\Big(\lr^2 \frac{\pl^3}{\rn^2}\Big),
\end{align}

where we have defined $\p_{t}(\ac_t) = \int_{\sz} \dots \int_{\sz} \prod_{k=0}^{t-1}\p(\ac_{t-k} | \ac_{t-k-1}, \params_0) \dee \ac_{1} \dots \dee \ac_{t-1}$.

\textbf{Step 3:} We now use the assumption that $\p_t$ converges weakly to our desired distribution $\p$, with a convergence rate that is $\CO(\frac{1}{t^2})$, to complete the proof.

Weak convergence, with the assumed rate, implies that for all bounded, measurable functions $f$, $|\mathbb{E}_{\p_t}(f) - \mathbb{E}_\p(f)| \leq \frac{M^*}{t^2}$, where $M^*$ is some positive constant. By assumption, the gradients are measurable and bounded, so we have:
\begin{align}
    \frac{\lr}{\rn}\sum_{t=0}^{\pl-1}\E_{\ac\sim \p_t}\big(\gest(\ac,\params)\big) &= \frac{\lr}{\rn}\sum_{t=\transient}^{\pl-1}\E_{\ac\sim \p_t}\big(\gest(\ac,\params)\big) + \frac{\lr}{\rn}\sum_{t=0}^{\transient-1}\E_{\ac\sim \p_t}\big(\gest(\ac,\params)\big) \nonumber \\
    &= \frac{\lr}{\rn}\sum_{t=\transient}^{\pl-1}\E_{\ac\sim \p}\big(\gest(\ac,\params)\big) + \frac{\lr}{\rn}\sum_{t=\transient}^{\pl-1}\CO\Big(\frac{M^*}{t^2}\Big) + \frac{\lr}{\rn}\sum_{t=0}^{\transient-1}\E_{\ac\sim \p_t}\big(\gest(\ac,\params)\big) \nonumber \\
    &= \frac{\pl-\transient}{\rn}\lr \E_{\ac\sim \p}\big(\gest(\ac, \params)\big) + \mathrm{\epsilon(\pl, \transient, \lr)},
\end{align}

where $\epsilon = \frac{\lr}{\rn}\sum_{t=\transient}^{\pl-1}\CO\Big(\frac{M^*}{t^2}\Big) + \frac{\lr}{\rn}\sum_{t=0}^{\transient-1}\E_{\ac\sim \p_t}\big(\gest(\ac,\params)\big)$ is the bias term. When $\pl$ is stochastic we can split up the sum in this way, using the deterministic variable $\transient$, because we originally conditioned on the fact that $\pl \geq \transient$ at the start of the proof. We now bound the bias. Observe that:
\begin{align}
    |\epsilon(\pl, \transient, \lr)| \leq \CO\Big(\frac{\lr}{\rn\transient}\Big) + \CO\Big(\frac{\lr \transient}{\rn}\Big) = \CO\Big(\frac{\lr \transient}{\rn}\Big),
\end{align}

where, for the inequality, we have used the fact that $\sum_{t=\transient}^\infty\frac{1}{t^2} \leq \frac{1}{\transient}$, and that $\gest$ is bounded. For the last equality one need simply note that $\transient > 1$. Lastly, observe that $\frac{T-\transient}{\rn} = \frac{\pl}{\rn} - \frac{\transient}{\rn}$, so the second term of this bias gets absorbed in the $\CO\big(\frac{\lr \transient}{\pl}\big)$ term.

Putting this together with the previous steps, we get:
\begin{align}
    \E\big[\gs(\ac_0, \params_0)\big|\pl] = \lr\frac{\pl}{\rn}\E_{\ac\sim \p}\big(\gest(\ac, \params)\big) + \CO\Big(\frac{\lr \transient}{\rn}\Big) + \CO\Big(\lr^2 \frac{\pl^3}{\rn^2}\Big).
\end{align}

And, finally, taking the expected value w.r.t. phase length yields
\begin{align}
    \E\big[\gs(\ac_0, \params_0)\big| \pls = 1] &= \lr\frac{\E(\pl | \pls = 1)}{\rn}\E_{\ac\sim \p}\big(\gest(\ac, \params)\big) + \CO\Big(\frac{\lr \transient}{\rn}\Big) + \CO\Big(\lr^2 \frac{\E(\pl^3 | \pls = 1)}{\rn^2}\Big) \nonumber \\
    &= \lr\E_{\ac\sim \p}\big(\gest(\ac, \params)\big) + \CO\Big(\frac{\lr \transient}{\rn}\Big) + \CO\Big(\lr^2 \rn \Big),
\end{align}

where we have used the results on the moments of $T$ from \S \ref{sec:moments-of-pl}. We note there that the $\CO(\lr\transient/\rn)$ term from the preliminary on stochastic phase length gets absorbed in the first bias term, for the stochastic phase length, yielding a convergence rate that is of the same order as with deterministic phase lengths. 

If we choose $\frac{1}{\lr}$ to be very large, $\rn$ to be less large, and $\transient$ to be less, but still sufficiently, large, we find that we can make the sum of the gradient steps during each phase arbitrarily close to being unbiased. For example, the scaling proposed in the theorem would be sufficient to achieve this.
\end{proof}

\subsection{Assumptions of Proof of Theorem \ref{th1}}

Finally, we comment on the validity of the assumption that $\extra(x)$ is integrable, and that the Markov chain exhibits weak convergence that is order $\frac{1}{t^2}$. The other assumptions we believe to be fairly mild restrictions on the regularity of the system.

For the assumptions on $\extra(x)$, we simply note that these will hold in cases where the functions involved are well-defined and $\sz$ is finite.

While we cannot prove that convergence to the stationary distribution is $\CO(\frac{1}{t^2})$ for every possible Markov chain we believe this holds for many Markov chains of interest. In particular, it holds for every finite state-space (finite $\sz$) Markov chain that is both irreducible and aperiodic. Every Markov chain satisfying these three desiderata (irreducible, aperiodic, finite state-space) also satisfies the following convergence rate result \cite{rosenthal1995convergence}:
\begin{align}
||p_t - p_\infty||_{\mathrm{tv}} \leq Ct^{l-1} \lambda^{t - l + 1}
\end{align}

where $||\cdot||_\mathrm{tv}$ is the total variation distance for probability measures, $\lambda$ is the second largest eigenvalue of the transition matrix for the Markov chain, and satisfies $\lambda < 1$, $C$ is a positive constant, and $l$ is the size of the largest Jordan block of the transition matrix. First, observe that convergence in total variation distance implies weak convergence. Second, this rate of convergence, which we define to be $r_1(t)$, is easily shown to be faster than $\CO(\frac{1}{t^2})$ = $r_2(t)$. To do so, it suffices to show that $\frac{r_1(t)}{r_2(t)} \to 0$. The ratio can be arranged as follows:
\begin{align}
\frac{r_1(t)}{r_2(t)} &= C_0 t^{l+1}\lambda^{t - l + 1} \nonumber \\ &= C_0 \frac{t^{l+1}}{\lambda^{l - t - 1}}.
\end{align}

Both numerator and denominator converge to $\infty$ as $t \to \infty$, and both are continuously differentiable $l + 1$ times, so we can apply L'Hopital's rule. Applying this rule and differentiating the top and bottom expressions $l + 1$ times gives:
\begin{align}
\lim_{t \to \infty}\frac{r_1(t)}{r_2(t)} = \lim_{t \to \infty} C_0 \frac{(l+1)!}{\lambda^{l-t-1}\big(-\ln(\lambda)\big)^{l+1}} = 0,
\end{align}

where the last equality follows because $\lambda < 1$.

\textbf{Remark:} We note that we have used the rather unsophisticated approach of simply Taylor expanding the parameter values during a wake/sleep phase around $\theta_0$, the value at the start of that phase, yielding a slowly converging error term. We expect that with more sophisticated methods, for example those outlined in \cite{yin2005discrete}, one might be able to establish a faster converging error rate.

\textbf{Remark 1:} The covariance requires an intractable integral to evaluate, so we once again appeal to MC sampling w.r.t $v$ and $h$ to evaluate it.

\textbf{Remark 2:} A minor complication with the above is that one requires more than one MC sample, w.r.t $v$ and $h$, to approximate the covariance function and, as such, one cannot use single sample MC sampling to estimate this gradient approximation.

\subsection{Multivariate Mean Value Theorem}
\label{sec:mvt}

In this section we include the multivariate mean value theorem as it features heavily in the proof of Theorem \ref{th1}.

\begin{theorem}
    Assume $f : \mathbb{R}^d \mapsto \mathbb{R}$ is a differentiable function. Then

    \begin{align}
        f(x + h) = f(x) + \nabla f(\xi) \cdot h
    \end{align}

    where $x, h \in \mathbb{R}^d$, and $\xi$ is a point on the line between $x$ and $x+h$.
\end{theorem}

We omit the proof as it can be found in any standard textbook. The following lemma follows from the above, and is used throughout the proof of Theorem \ref{th1}.

\begin{lemma}
    Assume that $f$ is $L$-Lipschitz (gradients of $f$ are bounded by $L$), and that $h \leq M \; \forall \: i$. Then it follows from the Mean Value Theorem that

    \begin{align}
        f(x + \eta h) = f(x) + \CO(\eta)
    \end{align}
    \label{mvt-corr}
\end{lemma}

\subsection{Convergence Rate of Error Induced by Stochastic Phase Lengths}
\label{sec:prob-app}

We wish to derive the rate at which this term converges to zero:

\begin{equation}
        \big(\E[\gs(\ac_0, \params)|\pls = 0] - \E[\gs(\ac_0, \params)|\pls = 1]\big)\p(\pls = 0).
        \label{eq-stoch-bias}
\end{equation}

We first consider the expression inside the brackets. Using the assumption that the gradients are bounded we get the following

\begin{equation}
    \begin{aligned}
        &\big(\E[\gs(\ac_0, \params)|\pls = 0] - \E[\gs(\ac_0, \params)|\pls = 1]\big) \leq \frac{A_1\transient\lr}{\rn} + \frac{A_2\E(\pl|\pls=1)\lr}{\rn} = \CO(\lr),
    \end{aligned}
\end{equation}

where $A_1$ and $A_2$ are constants. For the second term, we note that $\p(\pls=0) = 1 - \p(\pls=1) = 1 - \p(\trv = 0)^\transient = 1 - (1 - \trvp)^\transient$. By the definitions $\transient = \rn^m$ and $\trvp = \frac{1}{\rn}$, we get

\begin{equation}
    \p(\pls=0) = 1 - \Big(1 - \frac{1}{\rn}\Big)^{\rn^m}.
\end{equation}

Because $1$ is a constant, $\p(\pls = 0)$ goes to zero at the same rate that $(1 - \frac{1}{\rn})^{\rn^m}$ goes to $1$. We will now derive this rate.

\begin{equation}
    \begin{aligned}
        \Big(1 - \frac{1}{\rn}\Big)^\transient &= \exp\Big(\transient\log\Big[1 - \frac{1}{\rn}\Big]\Big) = \exp\Big(- \transient\Big[\frac{1}{\rn} + \CO\Big(\frac{1}{\rn^2}\Big)\Big]\Big) \\ 
        &= \exp\Big(- \frac{\transient}{\rn}\Big)\exp\Big( - \CO\Big(\frac{\transient}{\rn^2}\Big)\Big) = \Big[1 - \frac{\transient}{\rn} + \CO\Big(\frac{\transient^2}{\rn^2}\Big)\Big]\Big[1 - \frac{\transient}{\rn^2} + \CO\Big(\frac{\transient^2}{\rn^4}\Big)\Big] \\
        & = 1 - \frac{\transient}{\rn} + K_0(\transient, \rn).
    \end{aligned}
\end{equation}

Here, we have expanded the inner logarithm about $1$, given that $1 - \frac{1}{\rn}$ goes to $1$, and the exponential about $0$, given that $\transient\log\Big[1 - \frac{1}{\rn}\Big]$ goes to zero (because $\transient = \rn^m$ by definition). $K_0$ is thus a function purely of higher order terms than $\frac{\transient}{\rn}$. Thus, with leading order, $\p(\pls = 0)$ goes to zero with rate $\frac{\transient}{\rn}$.

Combining this result with the previous one gives that Equation \ref{eq-stoch-bias} is order $\CO\big(\frac{\lr\transient}{\rn}\big)$.

\subsection{Moments of $\pl$}
\label{sec:moments-of-pl}

Assume that for every time step on a discrete timeline one evaluates the Bernoulli random variable $\trv$ with parameter $\trvp << 1$. Assuming that, at time $t_0$, $\trv = 1$, one can define the random variable $\pl$ that is the number of time steps until the next observance of $\trv=1$. This is precisely the way that the stochastic phases are defined in the theorem. One can solve recursively for the moments of this random variable, using the law of total expectation and conditioning on the event that $\pl=1$. Define $\rn = \frac{1}{\trvp}$. Then the first three moments are found to be:

\begin{equation}
\begin{aligned}
    \E(\pl) &= \rn \\
    \E(\pl^2) &= 2\rn^2 - \rn \\
    \E(\pl^3) &= 6\rn^3 - 6\rn^2 - \rn.
\end{aligned}
\end{equation}

In the proof of Theorem \ref{th1}, and it's deterministic analog (see \S \ref{sec:th1d-app}), we are particularly interested in the moments conditioned on $\pls=1$. Because the Bernoulli random variable that decides whether the phase ends is sampled independently at each step, we get the following relations:

\begin{align}
    \E(\pl|\pls=1) &= \E(\pl+\transient) = \rn + \transient \\
    \E(\pl^2|\pls=1) &= \E[(\pl+\transient)^2] = \rn^2 + 2\rn\transient + \transient^2 \\
    \E(\pl^3|\pls=1) &= \E[(\pl+\transient)^3] = \rn^3 + 3\rn^2\transient + 3\rn\transient^2 + \transient
\end{align}

Importantly, these first and third moments are of leading order $1$, $2$ and $3$ w.r.t $\rn$. These observations are used in the proof of Theorem \ref{th1} and Theorem \ref{th1d-app}.

\section{Analog of Theorem \ref{th1} for Deterministic Equilibrium Dynamics}
\label{sec:th1d-app}

\begin{theorem}
\label{th1d-app}

Consider a single phase of equilibrium-based learning where the phase is either fixed to value $\rn$ or ends stochastically at each step with small probability $\frac{1}{\rn}$ and, instead of waiting until the end of the phase to do an update with learning rate $\lr$, an update is performed at every step of the equilibrium dynamics with learning rate $\frac{\lr}{\rn}$. Assume that equilibrium dynamics evolve for fixed $\params$ according to the homogeneous discrete time dynamical system given by the map $F(\ac, \params)$. Further assume that $\gest(\params)$ is differentiable and bounded with bounded partial derivatives for all network state values, and that $F$ is differentiable w.r.t. $\ac$ and $\params$, also with bounded partial derivatives. Finally, assume that the equilibrium dynamics converge to a fixed point at a rate at least as fast as $\CO(\frac{1}{t^2})$. Then the bias of the learning updates are of the following order:
\begin{align}
    \CO\Big(\frac{\lr \transient}{\rn}\Big) + \CO(\lr^2),
\end{align}

where $\transient = \rn^m$ for $0 < m < 1$ and $\CO$ is big-O notation.
\end{theorem}

\begin{proof}
The proof has the same structure as that of Theorem \ref{th1} and, accordingly, proceeds in three steps preceded by a preliminary comment on stochastic phase length.

\textbf{Comment on Stochastic Phase Length}
Exactly the same as the comment in the proof of Theorem \ref{th1}.

\textbf{Step 1:} Exactly the same as ``step 1" from the proof of Theorem \ref{th1}.

\textbf{Step 2:} This step uses a proof by induction. Following step 1, we have

\begin{equation}
    \begin{aligned}
        \E\big[\gs(\ac_0, \params_0)|\pl\big] = \frac{\lr}{\rn}\sum_{t=0}^{\pl-1}\gest(\tac_t, \params_0) + \CO\Big(\lr^2\frac{\pl^2}{\rn^2}\Big),
    \end{aligned}
    \label{th1d-1}
\end{equation}

where $\{\tac_t\}_{t=0}^{t=\pl-1}$ are the dynamics perturbed by the always learning updates to $\params$. We wish to demonstrate that this represents merely a negligible perturbation of the homogeneous dynamics $\{\ac_t\}_{t=0}^{t=\pl-1}$. Let us assume that $F(\ac, \params)$ is differentiable w.r.t. $\ac$ and $\params$ and that all partial derivatives are bounded by $L$ (thus, that it is $L$-Lipschitz). Then we have

\begin{equation}
    \begin{aligned}
        \tac_0 &= \ac_0 \\
        \tac_1 &= F(\tac_0, \params_0) = F(\ac_0, \params_0) = \ac_1 \\
        \tac_2 &= F(\tac_1, \params_1) = F\Big(\ac_1, \params_0 + \frac{\lr}{\tau}\gest(\ac_0, \params_0)\Big) = F(\ac_1, \params_0)  + \CO\Big(\frac{\lr}{\tau}\Big) = \ac_2 + \CO\Big(\frac{\lr}{\rn}\Big), 
    \end{aligned}
    \label{eq-basecase}
\end{equation}

where we used the multivariate mean value theorem corollary Lemma \ref{mvt-corr} and the assumption that $\gest$ is bounded for the third equality in the final line. The final line of Equation \ref{eq-basecase} will be the ``base case" in our proof by induction. We make the following induction hypothesis:

\begin{equation}
    \begin{aligned}
        \tac_t &= \ac_t + \CO\bigg(\frac{(t-1)\lr}{\rn}\bigg).        
    \end{aligned}
    \label{eq-z-relationship}
\end{equation}

We now complete the proof by induction---via the ``induction step" (proof of the $t+1$ case)---to show that the equality given in Equation \ref{eq-z-relationship} holds for arbitrary $t$:

\begin{equation}
    \begin{aligned}
        \tac_{t+1} &= F(\tac_t, \params_t) = F\Bigg(\ac_t + \CO\bigg(\frac{(t-1)\lr}{\rn}\bigg), \params_0 + \frac{\lr}{\rn}\sum_{k=0}^{t-1}\gest(\tac_k, \params_k)\Bigg) \\
        &= F(\ac_t, \params_0) + \CO\bigg(\frac{t\lr}{\rn}\bigg).
    \end{aligned}
\end{equation}

In the second equality of the first line we used the induction hypothesis to relate $\tac_t$ to $\ac_t$; we also used the definition of the parameter dynamics. For the final line we again used the assumption that $F$ is differentiable and Lipschitz, and the assumption that $\gest$ is bounded. We also used the fact that the $\CO\Big(\frac{(t-1)\lr}{\rn}\Big)$ term is also $\CO\Big(\frac{t\lr}{\rn}\Big)$.

Assuming $\gest$ has bounded partial derivatives w.r.t. $\ac$ and applying Lemma \ref{mvt-corr} we get:

\begin{equation}
    \begin{aligned}
        \gest(\tac_t, \params_0) = \gest(\ac_t, \params_0) + \CO\bigg(\frac{t\lr}{\rn}\bigg),
    \end{aligned}
\end{equation}

so that, applying these results to Equation \ref{th1d-1}, we get:

\begin{equation}
    \begin{aligned}
        \E\big[\gs(\ac_0, \params_0)|\pl\big] &= \frac{\lr}{\rn}\sum_{t=0}^{\pl-1}\gest(\ac_t, \params_0) + \frac{\lr}{\rn}\sum_{t=0}^{\pl-1}\CO\bigg(\frac{t\lr}{\rn}\bigg) + \CO\Big(\lr^2\frac{\pl^2}{\rn^2}\Big) \\
        &= \frac{\lr}{\rn}\sum_{t=0}^{\pl-1}\gest(\ac_t, \params_0) + \CO\Big(\lr^2\frac{\pl^2}{\rn^2}\Big).
    \end{aligned}
\end{equation}

\textbf{Step 3:} This step proceeds in essentially the same way as step 3 of the proof of the Markov dynamics version. We leverage the assumption that $\ac_t$ converges to a fixed point of $\ac_\infty$ with a convergence rate that is $\CO(\frac{1}{t^2})$:

\begin{align}
    \frac{\lr}{\rn}\sum_{t=0}^{\pl-1}\gest(\ac_t,\params) &= \frac{\lr}{\rn}\sum_{t=\transient}^{\pl-1}\gest(\ac_t,\params) + \frac{\lr}{\rn}\sum_{t=0}^{\transient-1}\gest(\ac_t,\params) \nonumber \\
    &= \frac{\lr}{\rn}\sum_{t=\transient}^{\pl-1}\gest(\ac_\infty,\params) + \frac{\lr}{\rn}\sum_{t=\transient}^{\pl-1}\CO\Big(\frac{1}{t^2}\Big) + \frac{\lr}{\rn}\sum_{t=0}^{\transient-1}\gest(\ac_t,\params) \nonumber \\
    &= \frac{\pl-\transient}{\rn}\lr \gest(\ac_\infty, \params) + \mathrm{\epsilon(\pl, \transient, \lr)}.
\end{align}

On line two of the above, we used our assumptions on the differentiability and boundedness of the partial derivatives of $\gest$ w.r.t. $\ac$, to apply Lemma \ref{mvt-corr} yet again. We also used the variable $\transient$ defined in the theorem, and defined $\epsilon = \frac{\lr}{\rn}\sum_{t=\transient}^{\pl-1}\CO\Big(\frac{1}{t^2}\Big) + \frac{\lr}{\rn}\sum_{t=0}^{\transient-1}\gest(\ac,\params)$ as the bias term. Importantly, we have also taken $\transient$ large enough so that, for $t \geq \transient$, $\ac_t$ is sufficiently close to $\ac_\infty$ to allow us to expand $\gest$ around $\ac_\infty$. We now bound the bias term.

Observe that:
\begin{align}
    |\epsilon(\pl, \transient, \lr)| \leq \CO\Big(\frac{\lr}{\rn\transient}\Big) + \CO\Big(\frac{\lr \transient}{\rn}\Big) = \CO\Big(\frac{\lr \transient}{\rn}\Big),
\end{align}

where, for the inequality, we have used the fact that $\sum_{t=\transient}^\infty\frac{1}{t^2} \leq \frac{1}{\transient}$, and that $\gest$ is bounded. For the last equality one need simply note that $\transient > 1$. Lastly, observe that $\frac{T-\transient}{\rn} = \frac{\pl}{\rn} - \frac{\transient}{\rn}$, so the second term of this bias gets absorbed in the $\CO\big(\frac{\lr \transient}{\pl}\big)$ term.

Putting this together with the previous steps, we get:
\begin{align}
    \E\big[\gs(\ac_0, \params_0)\big|\pl] = \lr\frac{\pl}{\rn}\gest(\ac_\infty, \params) + \CO\Big(\frac{\lr \transient}{\rn}\Big) + \CO\Big(\lr^2 \frac{\pl^2}{\rn^2}\Big).
\end{align}

And, finally, taking the expected value w.r.t. phase length yields
\begin{align}
    \E\big[\gs(\ac_0, \params_0)\big|\pls = 1] &= \lr\frac{\E(\pl|\pls=1)}{\rn}\gest(\ac_\infty, \params) + \CO\Big(\frac{\lr \transient}{\rn}\Big) + \CO\Big(\lr^2 \frac{\E(\pl^2|\pls=1)}{\rn^2}\Big) \nonumber \\
    &= \lr\gest(\ac_\infty, \params) + \CO\Big(\frac{\lr \transient}{\rn}\Big) + \CO(\lr^2),
\end{align}

where we have used the results on the moments of $T$ from \S \ref{sec:moments-of-pl}. If we choose $\lr$ to be large, $\rn$ to be large, and $\transient$ to be less large than $\rn$, but still sufficiently large for convergence of the dynamics, we find that we can make the sum of the gradient steps during each phase arbitrarily close to being unbiased. For example, the scaling proposed in the remark after Theorem \ref{th1} would be sufficient to achieve this.

\end{proof}

\textbf{Remark on Difference Between Stochastic and Deterministic Dynamics} Interestingly, we find that the bias for deterministic dynamics is lower than the bias for stochastic dynamics. Furthermore, for deterministic dynamics we don't need the learning rate to be of lower order than $\rn$ or $\transient$. The reason for these phenomena is the change from a second bias term of $\CO(\lr^2\pl)$ for stochastic dynamics to $\CO(\lr^2)$ for deterministic dynamics. Mathematically, this change occurs because of the difference in how one removes the inhomogeneity in the dynamics for the stochastic versus the deterministic case: one has to expand a \emph{product} of functions (inside an integral) in the stochastic case, compared to expanding within a \emph{composition} of functions in the deterministic case.

It is an open question whether this difference in bias between deterministic and stochastic dynamics is a fundamental property or is purely a consequence of the techniques used in the proof. In particular, we wonder whether the stochastic bias could be reduced to that of the deterministic dynamics if one were to use techniques from operator theory \cite{lasota1998chaos, yin2005discrete}. We leave this to future work.

\section{Algorithm Performance as a Function of Positive Phase Probability for MNIST}
\label{sec:mnist-fig}

\begin{figure}[ht!]
\vskip 0.2in
\begin{center}
\centerline{\includegraphics[width=0.6\columnwidth]{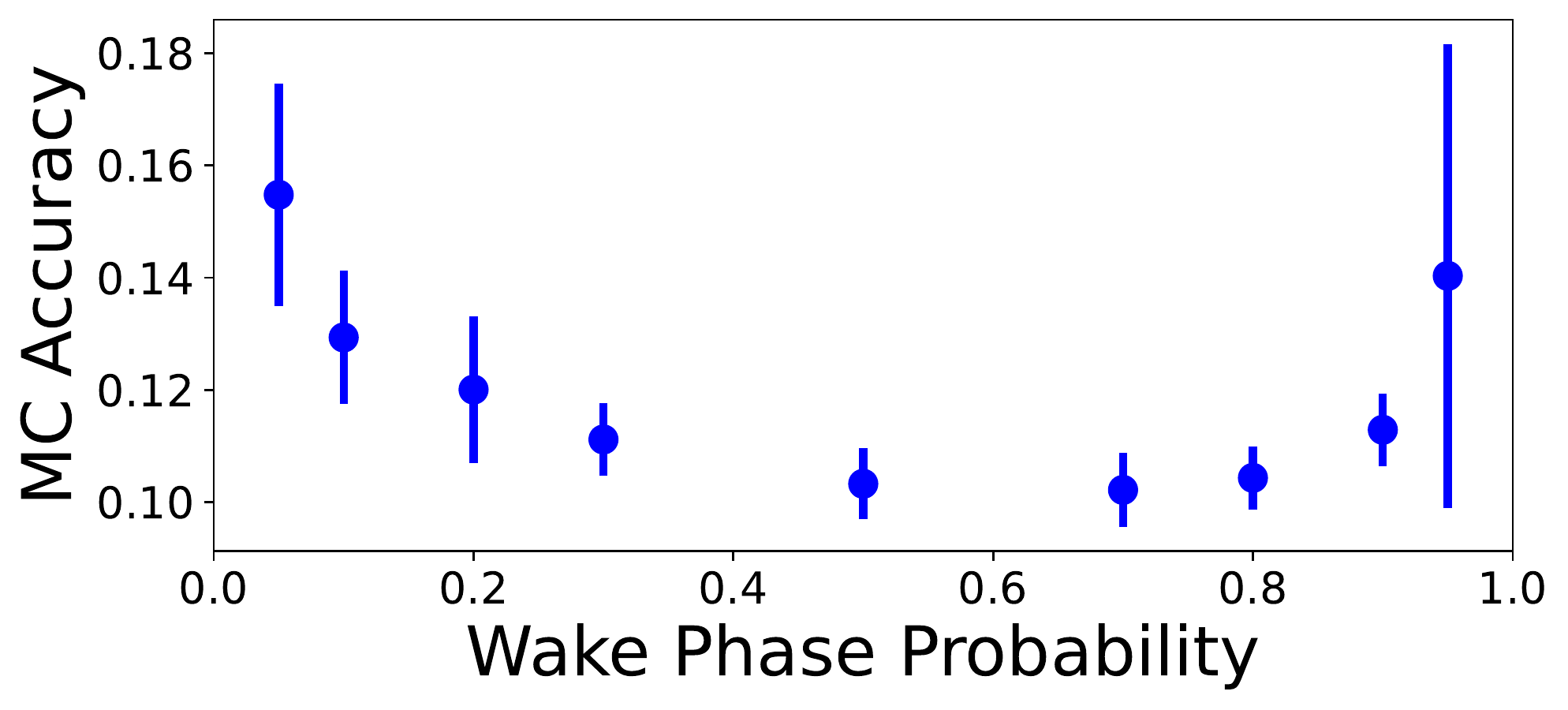}}
\caption{Same as Fig.\ref{res-fig2}.B except only the AoL values are shown, and the model is an RBM trained on binarized MNIST. Observe that the minimal value of $\prvp$ is higher than $0.5$. For this experiment, the RBM was trained on the concatenation of the images and their labels and the error plotted is classification error calculated, for a given image, by presenting the image without the label and selecting the predicted image label as the one with highest mean over a set of samples from the RBM. The RBM had 794 visible units and 128 hidden, and was trained with a learning rate of $0.005$.}
\label{res-fig-app}
\end{center}
\vskip -0.2in
\end{figure}

\section{Experimental Details}
\label{exp_details}
\paragraph{General:}
For the experiments in Fig.\ref{res-fig1-3}, and all distinct models, learning rates were selected by performing a line search over ten values, equally spaced on a log10 scale, to find the learning rate with lowest end-of-training training error. The learning rates in Fig.\ref{res-fig2} were simply set to $0.05$ and $0.0025$, for ISD AoL and ISD respectively, as the goal of this figure was not to compare algorithms and these learning rates were found to yield robust learning on a reasonable time frame. Network sizes were fixed to provide easy comparison across models.
Data was divided into segregated train and test sets for the MNIST experiments, but the test set was only used with the forward-forward algorithm. The BAS dataset is small and the entire, ground-truth data was trained on; therefore there was no test/training split in this case.
The code used for the experiments can be found at the following Github repository: \hyperlink{https://github.com/zek3r/ICML2023}{https://github.com/zek3r/ICML2023}.

\paragraph{RBM Parameters:}
For all experiments on the BAS dataset we trained a RBM with 16 hidden units and 16 visible units. The max and min values for the learning rate line search for Fig.\ref{res-fig1-3} were $0.04$ and $0.001$. Phase length was defined $\pl = 100$ for ISD, 100 for ISD AoL with fixed phase length, and was assigned a mean of 150 for the random phase length version. All networks were trained using vanilla SGD with no regularization. Initialization of all parameters was to white noise with standard deviation of $0.01$. All training runs in Fig.\ref{res-fig1-3} were $10^5$ steps long, except for CD1 which was $10^7$ gradient steps long (and failed to converge). Here, a step is defined as a gradient step for CDK and regular ISD, and a phase for ISD AoL and ISD AoL Random $\pl$. The reason that phase was used instead of gradient step for these latter two algorithms is that these algorithms performed a gradient step at every step of MCMC sampling but, as mentioned in Section \ref{sec:3.2}, learning rates were scaled by the mean phase length so that the sum of gradient steps taken during a phase was approximately equal to the single gradient step taken in regular ISD.

\paragraph{FF Parameters:}
With the forward-forward algorithm, we trained a network with two hidden layers of 500 units each. ADAM was used for both the standard algorithm and the ISD algorithm. The max and min values for the learning rate line search were $0.05$ and $0.002$. All training runs in Fig.\ref{res-fig1-3}, for all algorithms, were $120000$ gradient steps long.

\section{Ethics}

\paragraph{Energy Consumption:}
Given the scale of the experiments, we believe the energy consumption of this project was negligible. Moreover, the experiment was performed on an electrical grid that primarily uses hydro-electricity from long-established generating stations, meaning that the greenhouse gas emissions were likely close to zero.

\paragraph{Other Considerations:}
The main application of this work is to build basic scientific knowledge to aid computational neuroscience modelling and neuromorphic hardware development. As such, societal implications could eventually encompass clinical neuroscientific interventions or reductions in the energy consumption of deep learning methods, which we believe are largely positive. However, the commodification and weaponization of AI technology is already having negative impacts on society (e.g. use of facial recognition for discrimination, or the accentuating of wealth inequality and job losses). The authors are doing their best to stay abreast of these problems, and take steps to help combat them. For example, to combat commodification and centralization of wealth by making work publicly available.

\end{document}